\documentclass{article} % For LaTeX2e
\usepackage{iclr2027_conference,times}

\usepackage{amsmath,amsfonts,bm}

\def\eqref#1{equation~\ref{#1}}

\def\1{\bm{1}}

\DeclareMathAlphabet{\mathsfit}{\encodingdefault}{\sfdefault}{m}{sl}
\SetMathAlphabet{\mathsfit}{bold}{\encodingdefault}{\sfdefault}{bx}{n}

\usepackage[utf8]{inputenc}
\usepackage[T1]{fontenc}
\usepackage{microtype}
\usepackage{graphicx}
\usepackage{subcaption}
\usepackage{booktabs}
\usepackage[table]{xcolor}
\usepackage{hyperref}
\hypersetup{hidelinks}
\usepackage{url}

\usepackage{algorithm}
\usepackage{algorithmic}

\usepackage{amsmath}
\usepackage{amssymb}
\usepackage{mathtools}
\usepackage{amsthm}
\usepackage{etoolbox}
\usepackage{bm}
\usepackage{empheq}
\usepackage{wrapfig}
\usepackage{placeins}
\usepackage{needspace}

\usepackage[capitalize,noabbrev]{cleveref}

\usepackage[most]{tcolorbox}

\newtcolorbox{graybox}{
    colback=gray!8,
    colframe=gray!32,
    boxrule=0.3pt,
    arc=2pt,
    left=0pt,
    right=0pt,
    top=0pt,
    bottom=0pt
}

\newtheoremstyle{plaininline}
  {\topsep}{\topsep}{\itshape}{}{\bfseries}{.}{ }{}

\newtheoremstyle{definitioninline}
  {\topsep}{\topsep}{}{}{\bfseries}{.}{ }{}

\newtheoremstyle{remarkinline}
  {\topsep}{\topsep}{}{}{\itshape}{.}{ }{}

\theoremstyle{plaininline}
\newtheorem{theorem}{Theorem}[section]
\newtheorem{proposition}[theorem]{Proposition}
\newtheorem{lemma}[theorem]{Lemma}
\newtheorem{corollary}[theorem]{Corollary}

\theoremstyle{definitioninline}

\newtheorem{problem}[theorem]{Problem}
\newtheorem{remark}[theorem]{Remark}

\theoremstyle{remarkinline}

\AtBeginEnvironment{theorem}{\leavevmode}
\AtBeginEnvironment{proposition}{\leavevmode}
\AtBeginEnvironment{lemma}{\leavevmode}
\AtBeginEnvironment{corollary}{\leavevmode}
\AtBeginEnvironment{definition}{\leavevmode}
\AtBeginEnvironment{assumption}{\leavevmode}
\AtBeginEnvironment{problem}{\leavevmode}
\AtBeginEnvironment{remark}{\leavevmode}

\title{Distributionally Robust Schr\"{o}dinger Bridge}

\author{%
\begin{minipage}[t]{\dimexpr\textwidth-2\tabcolsep\relax}
\centering\normalfont
\textbf{Jinhwan Sul}\quad
\textbf{Panagiotis Theodoropoulos}\quad
\textbf{Vincent Pacelli}\\[2pt]
\textbf{Jaemoo Choi}\quad
\textbf{Evangelos Theodorou}\\[7pt]
Department of Aerospace Engineering \\
Georgia Institute of Technology, Atlanta, GA, USA\\[4pt]
{\small\texttt{\{jsul7,ptheodor3,vpacelli3,jchoi843,evangelos.theodorou\}@gatech.edu}}
\end{minipage}
}

\iclrfinalcopy
\makeatletter
\patchcmd{\@maketitle}{\lhead{Published as a conference paper at ICLR 2027}}{\lhead{}}{}{}
\makeatother
\renewcommand{\headrulewidth}{0pt}
\hypersetup{
  pdftitle={Distributionally Robust Schr\"{o}dinger Bridge},
  pdfauthor={Jinhwan Sul, Panagiotis Theodoropoulos, Vincent Pacelli, Jaemoo Choi, Evangelos Theodorou}
}

\begin{document}

\maketitle

\begin{abstract}
    Schrödinger bridge (SB) learns stochastic transport between prescribed initial and target distributions. When the initial distribution shifts at test time, the learned dynamics can fail to recover the target distribution. We introduce the Distributionally Robust Schrödinger Bridge (DRSB), which learns a single controller that accounts for uncertainty in the initial distribution. The DRSB objective consists of control energy and a KL penalty between the resulting terminal distribution and the target distribution. DRSB seeks a single controller that minimizes the worst-case value of this objective as the initial distribution varies within an ambiguity set around the nominal distribution. We derive an exact variational formulation of this objective and connect its fixed-terminal-cost subproblem to stochastic optimal control and distributionally robust optimization. This formulation motivates an alternating algorithm that updates the adversarial initial distribution, estimates the terminal log-density ratio, and trains the controller. We develop Wasserstein and Sinkhorn variants using stochastic control optimality conditions to approximate the gradients required for adversarial updates. Experiments on two-dimensional transport tasks and image-to-image translation show improved robustness to input perturbations relative to standard SB, with a tradeoff in nominal performance. On Gaussian mixture transport, Sinkhorn DRSB also achieves lower mean sliced Wasserstein distance than fixed-level noise augmentation at both tested unseen noise levels.
\end{abstract}

\section{Introduction}

Generative modeling is fundamentally a task of learning a transport map from a source distribution to a target distribution. While early approaches, such as GANs~\citep{goodfellow2014generative}, focused on learning static maps, recent approaches such as diffusion models~\citep{ho2020denoising,song2021score} learn dynamic transport maps via stochastic differential equations (SDEs).
The Schr\"{o}dinger Bridge (SB) problem generalizes standard diffusion models and likewise seeks an optimal transport mapping described by an SDE.

By definition, SB seeks a stochastic process that connects two arbitrary distributions and minimizes the KL divergence with respect to a reference stochastic process~\citep{schrodinger1931umkehrung,leonard2014survey}.
SB has been used in applications such as sampling~\citep{liu2025adjoint}, image-to-image translation~\citep{theodoropoulos2024feedback}, and image restoration~\citep{liu20232}.
However, the SB dynamics are trained for a prescribed initial distribution. When this distribution shifts at test time, the same dynamics produce a different terminal distribution, violating the original target marginal constraint. This motivates a distributionally robust SB formulation that accounts for uncertainty in the initial distribution.

In this paper, we formulate the \textbf{Distributionally Robust Schr\"{o}dinger Bridge (DRSB)} problem to achieve robust generation under initial distribution shifts.
The goal of DRSB is to find a single controller that transports every initial distribution in a prescribed set close to the target distribution. This set, called the ambiguity set, consists of distributions within a prescribed distance of the nominal initial distribution. 
We connect DRSB to stochastic optimal control through a terminal cost that penalizes mismatch between the induced terminal distribution and the target distribution. 
This cost takes the form of the log-density ratio between these distributions, which we approximate with a neural network to provide a training signal for the controller.
Since the induced terminal distribution depends on both the controller and the initial distribution, the terminal cost must be re-estimated as these change.

For a fixed terminal cost, the resulting subproblem has the standard distributionally robust optimization (DRO) form~\citep{mohajerin2018data,gao2023distributionally}, where the control variable serves as the DRO decision variable.
Strong duality results from the DRO literature provide a dual formulation of the inner worst-case maximization over the ambiguity set. 
We update the adversarial initial distribution using this dual formulation and train the controller via adjoint matching~\citep{domingo2024adjoint} to minimize the expected transport cost under the current adversarial distribution.
We alternate these updates with re-estimation of the terminal cost, training its network as a binary classifier that distinguishes the resulting terminal samples from target samples.

We develop Wasserstein and Sinkhorn variants of DRSB, with adversarial input updates derived from their DRO dual formulations.
Using an optimality condition from stochastic optimal control theory~\citep{kappen2005path}, we approximate cost gradients with controller evaluations.
The Wasserstein variant iteratively updates adversarial inputs.
For the Sinkhorn variant, a first-order approximation of the cost yields a Gaussian sampling rule for each nominal input and a closed-form update of the dual variable.
This approximation allows adversarial inputs to be sampled without iterative inner optimization.

We evaluate DRSB on 2D Gaussian-to-Gaussian and GMM-to-GMM transport tasks and on image-to-image translation.
The results show that DRSB improves robustness to noisy and shifted initial distributions compared with standard SB.
On the 2D transport tasks, Sinkhorn DRSB outperforms Wasserstein DRSB under most tested perturbation conditions.

The main contributions of this paper are summarized as follows.
\begin{itemize}
    \item We formulate Distributionally Robust Schr\"{o}dinger Bridge problems to address initial distribution shift.
    \item We derive an exact variational formulation of DRSB and connect its fixed-terminal-cost subproblem to SOC and DRO.
    \item We develop block update algorithms for the Wasserstein and Sinkhorn variants using SOC optimality conditions.
\end{itemize}

\section{Preliminaries}

\subsection{Schr\"{o}dinger Bridge under Initial Distribution Uncertainty}
The standard Schr\"{o}dinger bridge (SB) problem~\citep{schrodinger1931umkehrung,leonard2014survey} is
\begin{problem}[Schr\"{o}dinger Bridge]
\label{prob:SB}
\begin{subequations}
\begin{align}
    & \,\, \min_{u_t}\, \sigma^2\text{KL}(p^{u,\mu} \|p^{\text{base}}) = \min_{u_t}\mathbb{E}_{X \sim p^{u,\mu}}\left[ \int_0^1 \frac{1}{2}\| u^{\theta}_t(X_t)\|^2 dt\right], \\
    \text{s.t.} \quad & dX_t =  u^\theta_t(X_t) dt + \sigma dW_t, \quad X_0 \sim \mu(X_0), \quad X_1 \sim \nu(X_1). \label{eq:dyn}
\end{align}
\end{subequations}
\end{problem}
Here, $p^{u,\mu}$ denotes the path measure induced by \cref{eq:dyn} with the initial law $\mu$, and $p^{\text{base}}:= p^{u:=0}$ denotes the base path measure when $u_t:=0$. The base path measure has the same initial law $\mu$.
The equality assumes the conditions for Girsanov's theorem.
$X_t \in \mathbb{R}^d$ is the $d$-dimensional state vector at time $t$. The control $u^\theta: [0,1] \times \mathbb{R}^d \to \mathbb{R}^d$ is parameterized by $\theta$. $W_t$ is a standard $d$-dimensional Brownian motion, and $\sigma$ is a constant positive diffusion coefficient.
$p^{u,\mu}_0$ and $p^{u,\mu}_1$ are the marginal distributions of $p^{u,\mu}$ at times $t=0$ and $t=1$, respectively.
Then, by definition, $p_0^{u,\mu} = \mu$ always holds.
The Schr\"{o}dinger bridge (SB) problem seeks the optimal control by minimizing the KL divergence between $p^{u,\mu}$ and $p^\text{base}$ while transporting the initial data distribution $\mu$ to the target data distribution $\nu$.

\paragraph{Deployment with an unknown initial distribution.}
Let $u^\star$ solve Problem~\ref{prob:SB} for the nominal initial distribution $\mu$.
For this fixed controller, the Fokker--Planck equation defines the map $\Phi_{\text{SB}}$ from initial to terminal distributions,
\begin{equation}
    \Phi_{\text{SB}}: p^{u^\star, \mu}_0 \mapsto p^{u^\star,\mu}_1, \quad \Phi_{\text{SB}}(\mu)=\nu.
\end{equation}
At deployment, inputs may instead follow an unknown distribution $\mu_{\text{test}}\neq\mu$.
Applying the same controller need not preserve terminal matching, \emph{i.e.}
\begin{equation}
    \Phi_{\text{SB}}(\mu_{\text{test}}) \neq \nu \quad \text{in general}.
\end{equation}
We consider a deployment scenario where the test-time initial law is not known during training and a single learned controller must be used without test-time retraining.

\Cref{fig:concept} illustrates this deployment test setting using point clouds.
In \cref{fig:concept_sb}, blue circles represent samples from the nominal initial distribution $\mu$, and purple squares represent samples from a perturbed distribution $\mu_{\text{test}}$.
The same state-dependent dynamics act on both input clouds, but a change in the distribution of starting points can change the terminal cloud.
The nominal terminal distribution matches $\nu$, whereas the perturbed terminal distribution may deviate from this unchanged target.

We represent uncertainty in the test-time initial distribution by a prescribed set of plausible initial distributions, called an ambiguity set.
\Cref{fig:concept_drsb} illustrates the goal of learning one controller for this set.
Without a known probability law over deployment shifts, we aim to minimize the worst-case sum of control energy and terminal mismatch over the ambiguity set, using the same learned controller at deployment.

\begin{figure}[t]
    \centering
    \begin{subfigure}[b]{0.44\linewidth}
        \centering
        \includegraphics[width=\linewidth]{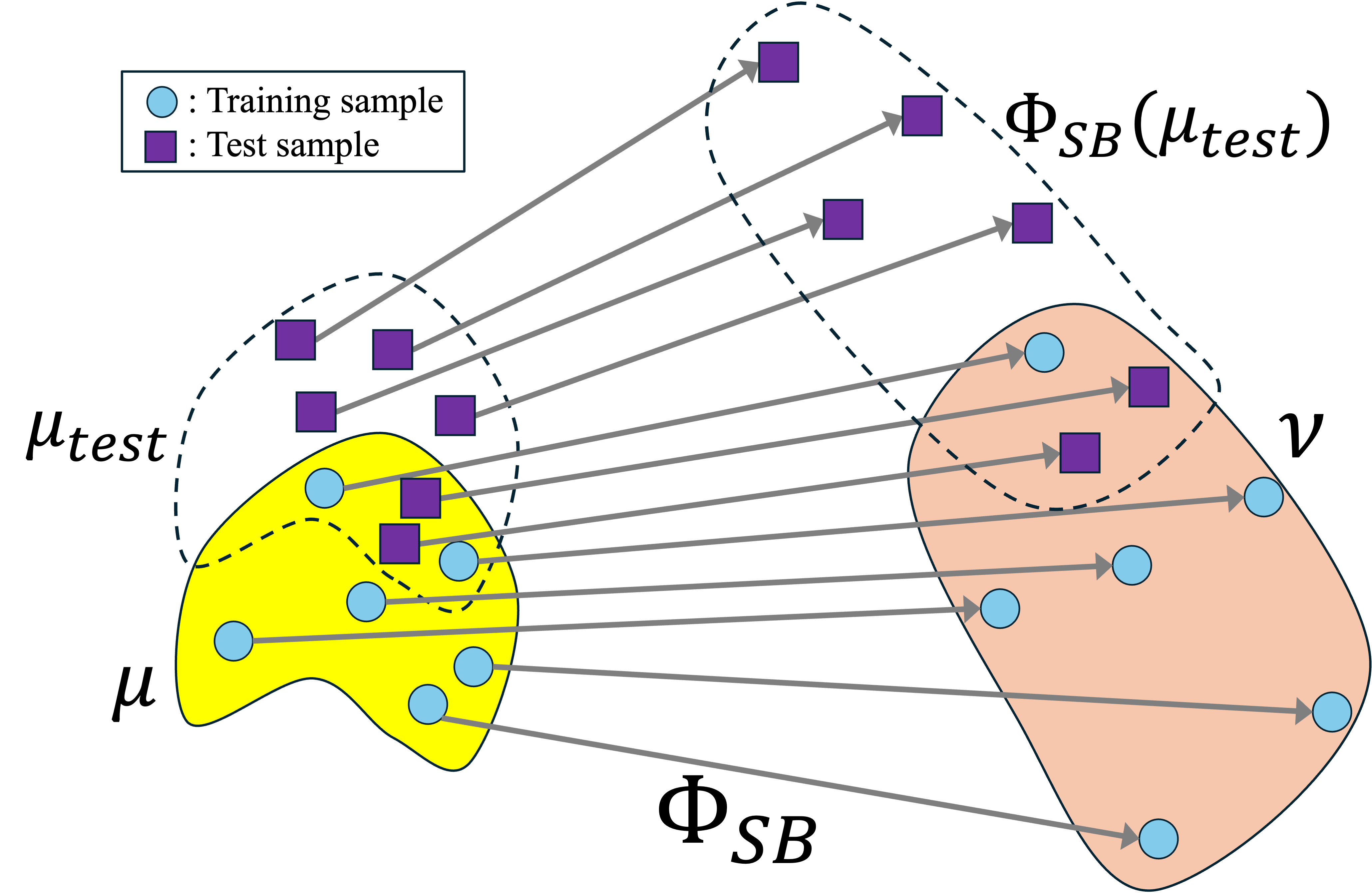}
        \caption{Standard SB}
        \label{fig:concept_sb}
    \end{subfigure}
    \hfill
    \begin{subfigure}[b]{0.48\linewidth}
        \centering
        \includegraphics[width=\linewidth]{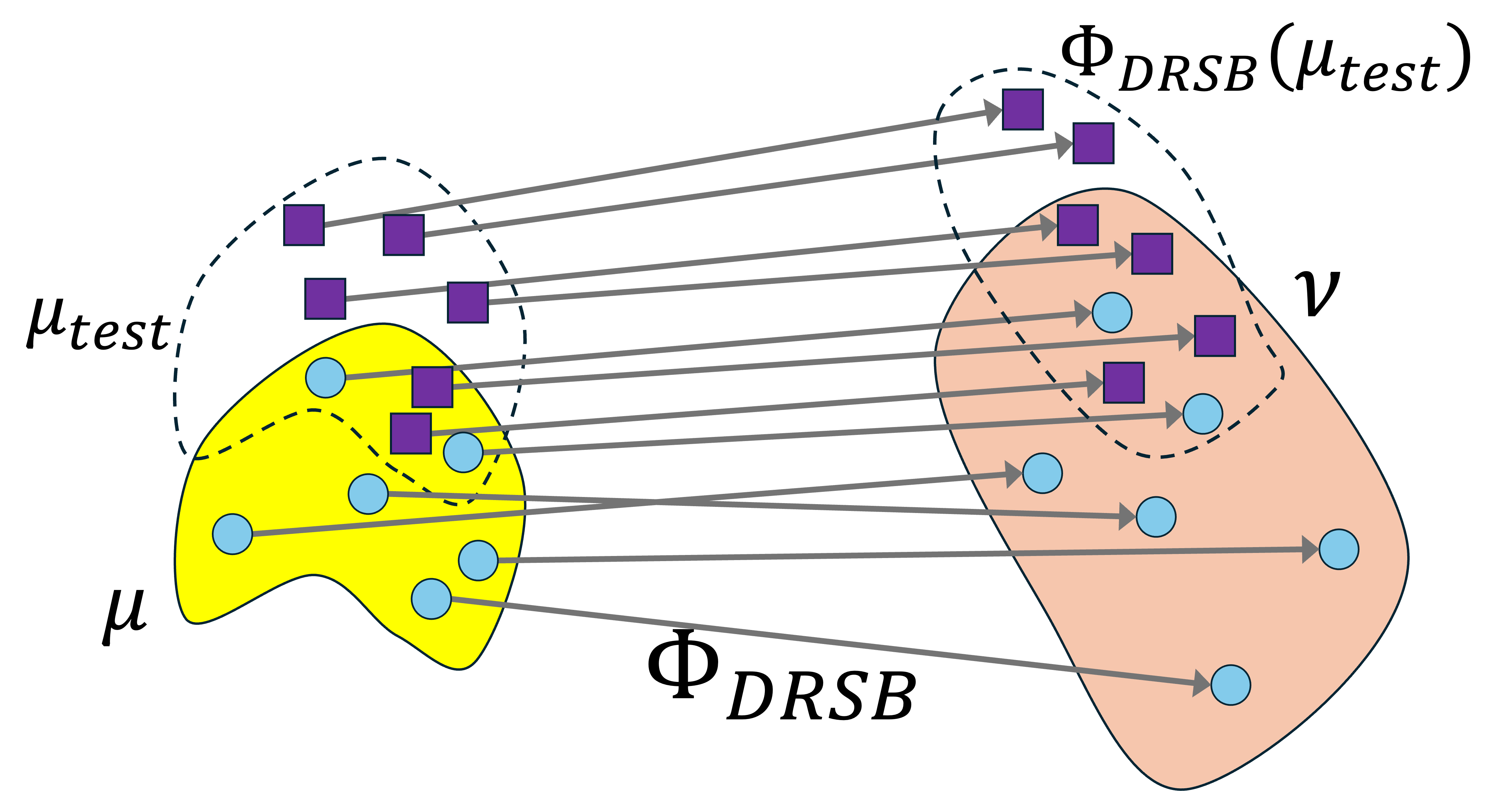}
        \caption{DRSB (ours)}
        \label{fig:concept_drsb}
    \end{subfigure}
    \caption{Deployment under initial distribution uncertainty. Blue circles denote nominal samples from $\mu$, and purple squares denote shifted test samples from $\mu_{\text{test}}$. The desired target $\nu$ is unchanged. (a) A standard SB controller matches $\nu$ for $\mu$, but this matching need not persist under an initial distribution shift. (b) DRSB trains one controller to reduce the worst-case sum of control energy and terminal mismatch over an ambiguity set of initial distributions.}
    \label{fig:concept}
\end{figure}

\subsection{Distributionally Robust Optimization}

In stochastic optimization, the underlying data distribution is often unknown, and only a few data samples are available.
DRO hedges against this uncertainty by optimizing under the worst-case distribution within a prescribed set of distributions called the ambiguity set~\citep{wiesemann2014distributionally}.
The DRO problem takes the minimax form
\begin{equation}
    \min_{u \in \mathcal{U}} \max_{\mu \in \mathfrak{M}_{\epsilon}(\hat{\mu})} \mathbb{E}_{X_0\sim \mu}[V(u, X_0)],
\label{eq:DRO}
\end{equation}
where the outer minimization selects the decision variable $u$, and the inner maximization identifies the worst-case distribution $\mu$ within the ambiguity set $\mathfrak{M}_{\epsilon}(\hat{\mu}) := \{ \mu : D(\mu, \hat{\mu}) \le \epsilon \}$, defined as the set of distributions within distance $\epsilon$ of the nominal distribution $\hat{\mu}$ under distributional distance $D$.

\section{Distributionally Robust Schr\"{o}dinger Bridge}
\label{sec:drsb}

\subsection{Distributional Robustness in SB and Connection to SOC}

We formulate and solve the Distributionally Robust Schr\"{o}dinger Bridge (DRSB) problem, which seeks a control that can robustly steer shifted test-time distributions $\mu_{\text{test}}$ toward the target distribution.
All proofs are provided in Appendix~\ref{appendix:drsb_proofs}.

The DRSB formulation is
\begin{graybox}
\begin{problem}[Distributionally Robust Schr\"{o}dinger Bridge]
\label{prob:drsb}
\begin{subequations}
\label{eq:drsb}
\begin{align}
    \min_{u} \max_{\mu \in \mathfrak{M}_{\epsilon}(\hat{\mu})}  & \mathbb{E}_{X \sim p^{u,\mu}} \left[ \int_0^1  \frac{1}{2}\| u^{\theta}_t(X_t)\|^2 dt\right] + \rho \text{KL}(p^{u,\mu}_1 \| \nu), \label{eq:drsb_obj} \\
    \text{s.t.} \quad & dX_t = u^{\theta}_t(X_t) dt + \sigma dW_t, \quad X_0 \sim \mu, \label{eq:drsb_dyn}\\
    & \mathfrak{M}_{\epsilon}(\hat{\mu}) = \{ \mu : D(\mu, \hat{\mu}) \le \epsilon \}. \label{eq:drsb_ambiguity}
\end{align}
\end{subequations}
\end{problem}
\end{graybox}
Here, the initial distribution $\mu$ is no longer assumed to be known exactly.
Instead, the inner maximization in \cref{eq:drsb_obj} finds the worst-case initial distribution from the ambiguity set $\mathfrak{M}_{\epsilon}(\hat{\mu})$.
$\hat{\mu}$ is the empirical distribution $\hat{\mu} = \frac{1}{N}\sum_{i=1}^N \delta_{\hat{X}_0^i}$, where $\delta_{\hat{X}_0^i}$ is a Dirac delta function at each sample point $\hat{X}_0^i$.
The ambiguity set is centered at $\hat{\mu}$ with radius $\epsilon$ under the distributional distance $D$.

DRSB addresses the choice of a single transport controller before the test-time input distribution is known.
We consider input perturbations for which the desired target distribution $\nu$ remains unchanged.
The control-energy term in \cref{eq:drsb_obj} retains the transport-cost criterion of SB, while the soft penalty $\rho\,\mathrm{KL}(p^{u,\mu}_1\|\nu)$ relaxes the hard terminal constraint $p^{u,\mu}_1 = \nu$ in \cref{eq:dyn}.
Optimality in Problem~\ref{prob:drsb} is defined by minimizing the worst-case sum of control energy and the weighted terminal KL divergence over the ambiguity set.
The resulting controller need not preserve the nominal SB transport or be optimal for each initial distribution separately.
Nominal performance may therefore be traded for lower worst-case cost.
This formulation provides a criterion for selecting one controller when adaptation to each deployment distribution is not possible.
\vspace{-2pt}
\begin{remark}
For a nonnegative distance with $D(\mu,\hat{\mu})=0$ if and only if $\mu=\hat{\mu}$, when $\epsilon = 0$, the ambiguity set reduces to $\mathfrak{M}_0(\hat{\mu}) = \{\hat{\mu}\}$, eliminating the inner maximization. When additionally $\rho \to \infty$, the KL penalty enforces $p^{u,\mu}_1 = \nu$ and we recover the standard SB in Problem~\ref{prob:SB}.
\end{remark}

To optimize the DRSB objective through stochastic optimal control, we require a terminal cost that represents the mismatch between $p_1^{u,\mu}$ and $\nu$.
We introduce a terminal cost neural network $g^\phi:\mathbb{R}^d \rightarrow \mathbb{R}$ parameterized by $\phi$ that approximates the log-density ratio between $p_1^{u,\mu}$ and $\nu$.
\vspace{-2pt}
\begin{proposition}[Terminal cost approximation]
\label{prop:g_guidance}
For a fixed admissible pair $(u,\mu)$, suppose $p_1^{u,\mu}$ and $\nu$ have positive densities and finite KL divergence. Suppose the function class contains their log-density ratio and a terminal cost function $g^\phi:\mathbb{R}^d \rightarrow \mathbb{R}$ is an optimal solution of the optimization problem
\begin{equation}
    g^\phi = \arg\min_{g} \mathbb{E}_{X_1\sim p^{u,\mu}_1}\left[\log \left(1+e^{-g(X_1)}\right)\right] + \mathbb{E}_{X_1\sim \nu}  \left[ \log \left(1+e^{g(X_1)} \right) \right].
\label{eq:g_discrim_opt}
\end{equation}
Then, $\quad g^\phi(x) = \log\left( \frac{p_1^{u,\mu}(x)}{\nu(x)} \right)$ and
\begin{equation}
    \mathrm{KL}(p_1^{u,\mu} \| \nu) = \mathbb{E}_{X_1 \sim p_1^{u,\mu}}[g^\phi(X_1)] - \log \mathbb{E}_{X_1 \sim \nu}[e^{g^\phi(X_1)}].
\label{eq:kl_variational}
\end{equation}
\end{proposition}
Proposition~(\ref{prop:g_guidance}) shows that the KL penalty in \cref{eq:drsb_obj} can be expressed in terms of the terminal cost $g^\phi$ using \cref{eq:kl_variational}. The terminal cost is obtained by solving (\ref{eq:g_discrim_opt}).
The discriminator $g^\phi$ distinguishes generated samples ($X_1 \sim p_1^{u,\mu}$) from target samples ($X_1 \sim \nu$).
As a terminal cost, $g^\phi$ guides the terminal distribution $p_1^{u,\mu}$ toward the target distribution $\nu$ through this representation of the KL penalty.
\Needspace{16\baselineskip}
\begin{theorem}[Variational formulation of DRSB]
\label{thm:DRSB_solved_with_DRO}
For each admissible pair $(u,\mu)$, apply the variational representation of the terminal KL divergence. Then Problem~\ref{prob:drsb} is equivalent to
\begin{equation}
\min_u\max_{\mu\in\mathfrak{M}_\epsilon(\hat{\mu})}\sup_g
\left\{\mathbb{E}_{X_0\sim\mu}[V(u,X_0;g)]
-\rho\log\mathbb{E}_{X_1\sim\nu}[e^{g(X_1)}]\right\},
\label{eq:drsb_variational}
\end{equation}
subject to \cref{eq:drsb_dyn,eq:drsb_ambiguity}, where
\begin{equation}
V(u,x;g):=\mathbb{E}_{X_{\cdot|0}\sim p^u_{\cdot|0}}\left[
\int_0^1\frac{1}{2}\|u_t^\theta(X_t)\|^2dt+\rho g(X_1)\,\middle|\,X_0=x\right].
\label{eq:drsb_cost}
\end{equation}
Here, $p^u_{\cdot|0}$ is the controlled path distribution conditioned on the initial state, and the supremum is over measurable $g$ for which the expectations are well defined and $0<\mathbb{E}_\nu[e^g]<\infty$.
The optimizing terminal cost depends on $(u,\mu)$ and cannot generally be fixed across the outer optimization.
\end{theorem}
For fixed $g$, the normalization term is constant in $(u,\mu)$, giving the subproblem
\begin{equation}
\min_u\max_{\mu\in\mathfrak{M}_\epsilon(\hat{\mu})}
\mathbb{E}_{X_0\sim\mu}[V(u,X_0;g)],\qquad
\text{s.t. \cref{eq:drsb_dyn,eq:drsb_ambiguity}}.
\label{eq:drsb_dro_form}
\end{equation}
We write $V(u,x)=V(u,x;g)$ when $g$ is held fixed.

From the SOC perspective, $V(u,x;g)$ in \cref{eq:drsb_cost} is the cost of a stochastic control problem with terminal cost $\rho g$. For fixed $(\mu,g)$, minimizing it gives a control adapted to that initial distribution and terminal cost~\citep{chen2021stochastic,liu2022deep}. A shift in the initial distribution can change the terminal marginal. DRSB instead minimizes the worst-case objective over the ambiguity set, retaining the optimization over $g$ in \cref{eq:drsb_variational}.

For fixed $g$, introducing a Lagrange multiplier $\lambda\geq0$ gives the dual of \cref{eq:drsb_dro_form}, with equality when strong duality holds:
\begin{equation}
    \min_u \min_{\lambda \geq 0} \lambda \epsilon + \max_{\mu} \left\{ \mathbb{E}_{X_0 \sim \mu}\left[ V(u,X_0) \right] - \lambda D(\mu, \hat{\mu}) \right\}, \quad \text{s.t. \quad (\ref{eq:drsb_dyn})}.
\label{eq:drsb_dro_dual}
\end{equation}

\subsection{Iterative Optimization Scheme for DRSB}

We use block updates for \cref{eq:drsb_variational} in the order $\mu^{(k)}\rightarrow g^{(k)}\rightarrow u^{(k)}$, holding the other two variables fixed at each step. First, with $(u^{(k-1)},g^{(k-1)})$ fixed, the adversarial initial distribution is updated through the DRO dual:
\begin{equation}
\begin{aligned}
\mu^{(k)},\lambda^{(k)}\leftarrow\min_{\lambda\geq0}\;\lambda\epsilon+\max_\mu
\bigl\{\mathbb{E}_{X_0\sim\mu}[V(u^{(k-1)},X_0;g^{(k-1)})]
-\lambda D(\mu,\hat{\mu})\bigr\}.
\end{aligned}
\label{eq:worst_case_dual}
\end{equation}
We simulate the dynamics with $u^{(k-1)}$ from $\mu^{(k)}$ and denote the resulting terminal marginal by $p_1^{(k)}=p_1^{u^{(k-1)},\mu^{(k)}}$. The terminal cost update is
\begin{equation}
    g^{(k)} = \arg\min_g \mathbb{E}_{X_1\sim p_1^{(k)}}\left[\log \left(1+e^{-g(X_1)}\right)\right] + \mathbb{E}_{X_1\sim \nu}  \left[ \log \left(1+e^{g(X_1)} \right) \right].
\label{eq:g_update}
\end{equation}
At the population optimum, this loss recovers the same log-density ratio as the KL variational objective. With $(\mu^{(k)},g^{(k)})$ fixed, the control subproblem is
\begin{equation}
u^{(k)}\leftarrow\arg\min_u\mathbb{E}_{X_0\sim\mu^{(k)}}[V(u,X_0;g^{(k)})],\qquad\text{s.t.}\quad (\text{\ref{eq:drsb_dyn}}).
\label{eq:dro_iterative}
\end{equation}
We use the adjoint matching objective~\citep{domingo2024adjoint}:
\begin{equation}
    u^{(k)} = \arg\min_u \mathbb{E}_{X_0 \sim \mu^{(k)}} \mathbb{E}_{t \sim \mathcal{U}[0,1]} \mathbb{E}_{X_t, X_1 \sim p^{\bar{u}}_{\cdot|0}}\left[ \| u^\theta(t, X_t) + \rho \nabla g^{(k)}(X_1) \|^2 \right],
\label{eq:AM_objective}
\end{equation}
where $p^{\bar{u}}_{\cdot|0}$ is the path measure induced by $\bar{u}=\texttt{stopgrad}(u^\theta)$, conditioned on the initial state. This objective regresses the control toward the negative terminal cost gradient. In practice, each block uses finitely many optimization steps. These updates neither exchange the minimax operators nor establish convergence of the full nested problem.

\noindent
\begin{minipage}[t]{0.48\linewidth}
\vspace{0pt}
The DRSB training procedure is summarized in \cref{alg:drsb_overview}.
Following the successive approximation method from robust control theory~\citep{bea1998successive}, we initialize $u^\theta$ with a pretrained standard SB model.
Steps 3--5 remain the same across all DRSB variants.
Only Step 2, the adversarial initial distribution update, depends on the choice of ambiguity set distance $D$.
The cost of Step 2 depends on this choice, while the total training cost also includes simulation and network updates.
\end{minipage}
\hfill
\begin{minipage}[t]{0.48\linewidth}
\vspace{0pt}
\begin{algorithm}[H]
\caption{Distributionally Robust SB}
\textbf{Require:} Drift $u^\theta$, terminal cost $g^\phi$, nominal distribution $\hat{\mu}$, target $\nu$.

\textbf{Initialize:} $u^\theta$ with a pretrained SB model, $g^\phi$, $\lambda$

\begin{algorithmic}[1]
\FOR{$k = 1, 2, \ldots, K$}
    \STATE $\mu^{(k)}, \lambda^{(k)} \leftarrow$ Update via (\ref{eq:worst_case_dual})
    \STATE Simulate (\ref{eq:drsb_dyn}) from $\mu^{(k)}, u^{(k-1)}$
    \STATE Update $g^{(k)}$ via (\ref{eq:g_update})
    \STATE Update $u^{(k)}$ via (\ref{eq:AM_objective})
\ENDFOR
\end{algorithmic}
\label{alg:drsb_overview}
\end{algorithm}
\end{minipage}

\subsection{Variants of DRSB}

Different choices of the distributional distance $D$ can lead to different DRSB variants.
We consider two variants based on the Wasserstein and Sinkhorn distances.
We use a fixed-terminal-cost SOC optimality condition to simplify the adversarial updates.
See Appendix~\ref{appendix:detail_drsb_vars} for details.

\paragraph{DRSB with Wasserstein Distance.} The following proposition shows the dual formulation of DRSB with Wasserstein distance.
\begin{proposition}[Dual form of Wasserstein DRSB]
The fixed-$g$ subproblem (\ref{eq:drsb_dro_form}) with the Wasserstein-2 ambiguity set $\mathfrak{M}^W_{\epsilon}(\hat{\mu}) = \{ \mu : W^2_2(\mu, \hat{\mu}) \le \epsilon \}$ admits strong duality~\citep{zhang2025short,gao2023distributionally}, and the corresponding dual problem~(\ref{eq:drsb_dro_dual}) specializes to
\begin{equation}
    \min_u \min_{\lambda \geq 0} \left\{ \lambda\varepsilon  -  \frac{1}{N}\sum_{i=1}^N \min_{X_0} \left\{ \lambda \| X_0-\hat{X}^i_0 \|^2 - V(u,X_0) \right\} \right\},  \quad
    \text{s.t. \quad (\ref{eq:drsb_dyn}).}
\label{eq:wasserstein_drsb}
\end{equation}
\end{proposition}
Since the nominal distribution is a discrete empirical distribution, the inner minimization must be solved for every sample.
Our implementation represents the adversary by one perturbed sample per nominal sample~\citep{gao2023distributionally}.

\cref{eq:wasserstein_drsb} is a nested optimization problem over $\lambda$ and the per-sample initial state $X_0$.
In practice, we use the value of $\lambda$ from the previous iteration when optimizing the worst-case samples, just as we use the controller from the previous iteration in \cref{eq:worst_case_dual}.
We optimize $X_0$ using gradient descent.
For a smooth value function and an optimal controller $u^*$ of the fixed-$g$ SOC subproblem, the optimality condition is~\citep{kappen2005path}:
\begin{equation}
    u^*_0(x) = - \nabla V(u^*, x).
\label{eq:optimal_control}
\end{equation}
Replacing $-\nabla V(u,x)$ by $u_0(x)$ uses one network evaluation. This is exact at the fixed-$g$ SOC optimum and is a surrogate for a controller after finite optimization. For outer iteration $k$, hold $u^{(k-1)}$ fixed and let $j$ index the inner updates:
\begin{equation}
    X^{i,(j+1)}_0 \leftarrow X^{i,(j)}_0 - \eta \left[u_0^{(k-1)}(X^{i,(j)}_0) + 2\lambda^{(j)}(X^{i,(j)}_0 - \hat{X}_0^i) \right], \quad \eta\text{ : step size}.
\label{eq:worst_case_gd}
\end{equation}
The dual update is
\begin{equation}
    \lambda^{(j+1)} \leftarrow \lambda^{(j)} - \zeta \left( \varepsilon - \frac{1}{N}\sum_{i=1}^N \|X^{i,(j+1)}_0 - \hat{X}^i_0\|^2 \right), \quad \zeta\text{ : step size},
\label{eq:lambda_update_rule}
\end{equation}
which is valid as a dual update only while $\lambda^{(j+1)}\geq0$.

\paragraph{DRSB with Sinkhorn Distance.}
The Sinkhorn distance~\citep{cuturi2013sinkhorn} between $\mu$ and $\hat{\mu}$ is
\begin{equation}
    \mathcal{W}_{\kappa}(\hat{\mu}, \mu) :=  \min_{\gamma \in \Gamma(\hat{\mu}, \mu)}  \left\{ \mathbb{E}_{(\hat{X}_0, X_0)\sim \gamma}\left[\|\hat{X}_0 - X_0\|^{2}  +  \kappa \log \left( \frac{d\gamma(\hat{X}_0, X_0)}{d\hat{\mu}(\hat{X}_0) dX_0} \right)  \right] \right\},
\end{equation}
where $\Gamma(\hat{\mu}, \mu)$ is the set of joint distributions between $\hat{\mu}$ and $\mu$, and $\kappa > 0$ is the entropy regularization parameter.
\begin{proposition}[Dual form of Sinkhorn DRSB]
The fixed-$g$ subproblem (\ref{eq:drsb_dro_form}) with the Sinkhorn ambiguity set $\mathfrak{M}^S_{\epsilon}(\hat{\mu}) = \{ \mu : \mathcal{W}_\kappa(\hat{\mu}, \mu) \le \epsilon \}$ admits strong duality~\citep{wang2025sinkhorn}, and the corresponding dual problem~(\ref{eq:drsb_dro_dual}) specializes to
\begin{equation}
    \min_u \inf_{\lambda > 0} \lambda \varepsilon + \frac{1}{N}\sum_{i=1}^N \lambda \kappa \log \int e^{ \left( V(u,X_0) - \lambda \|X_0 - \hat{X}^i_0\|^2 \right)/\lambda \kappa} dX_0, \quad \text{s.t. \quad (\ref{eq:drsb_dyn}),}
\label{eq:sinkhorn_drsb}
\end{equation}
The $\lambda=0$ boundary is interpreted through a limit. For $\lambda>0$ and a finite normalizing integral, the worst-case conditional distribution $\gamma^*(X_0 | \hat{X}^i_0)$ is obtained as
\begin{equation}
    \gamma^*(X_0 | \hat{X}^i_0) \propto e^{\left(V(u,X_0) - \lambda \|X_0 - \hat{X}^i_0\|^2\right)/\lambda\kappa}.
\label{eq:sink_worst_case}
\end{equation}
\end{proposition}
In contrast to the Wasserstein case, the worst-case distribution in Sinkhorn DRSB is supported on all of $\mathbb{R}^d$, even when the nominal distribution $\hat{\mu}$ is discrete.
We approximate this distribution by linearizing the fixed-$g$ cost.
\begin{corollary}[Gaussian worst-case distribution]
\label{corol:gaussian_worst}
Fix $(u,g)$ and assume the SOC optimality condition in \cref{eq:optimal_control} holds. Suppose $\varepsilon+\frac{\kappa d}{2}\log(\pi\kappa)>0$. Under the first-order approximation $V(X_0) \approx V(\hat{X}_0^i) \!+\! \nabla V(\hat{X}_0^i)\!^\top\!(X_0 \!-\! \hat{X}_0^i)$, where $V(x)\!=\!V(u,x;g)$, the linearized inner dual in \cref{eq:sinkhorn_drsb} is
\begin{equation}
    \inf_{\lambda > 0} \; \lambda \left( \varepsilon + \frac{\kappa d}{2} \log \pi \kappa \right) + \frac{1}{N}\sum_{i=1}^N \left[ V(\hat{X}_0^i ) + \frac{\|u^\theta_0(\hat{X}_0^i) \|^2}{4 \lambda} \right],
\label{eq:sink_drsb_approx}
\end{equation}
where the worst-case conditional distribution is
\begin{equation}
    \gamma^*(X_0 | \hat{X}^i_0) = \mathcal{N} \left( \hat{X}^i_0 - \frac{u^\theta_0(\hat{X}_0^i)}{2\lambda} \,, \frac{\kappa}{2} I \right),
\label{eq:sink_gaussian_worst}
\end{equation}
and the optimal dual variable admits the closed form:
\begin{equation}
    \lambda^* = \sqrt{\left( \varepsilon + \frac{\kappa d}{2} \log \pi \kappa \right)^{-1} \frac{1}{N}\sum_{i=1}^N \left[ \frac{\|u^\theta_0(\hat{X}_0^i) \|^2}{4} \right]}.
\label{eq:sdrsb_opt_lambda}
\end{equation}
\end{corollary}
Corollary~\ref{corol:gaussian_worst} gives closed-form sampling for the linearized inner problem.

\begingroup
\captionsetup{font=small,skip=4pt,singlelinecheck=false}
\setlength{\intextsep}{7pt plus 2pt minus 2pt}
\setlength{\columnsep}{11pt}
\newcommand{\experimenttable}{%
  \small
  \setlength{\tabcolsep}{3pt}%
  \renewcommand{\arraystretch}{1.05}%
}
\newcommand{\meanstd}[2]{\mbox{#1{\scriptsize$\pm$#2}}}
\newcommand{\finishwrappedpanel}{%
  \par
  \ifnum\value{WF@wrappedlines}>2\relax
    \vspace{\dimexpr\value{WF@wrappedlines}\baselineskip-2\baselineskip\relax}%
  \fi
  \WFclear
}
\FloatBarrier
\section{Experiments}
\suppressfloats[t]

\subsection{Robustness against noisy and shifted distributions on 2D transport tasks}

We demonstrate the robustness of our DRSB algorithm on the 2D Gaussian-to-Gaussian and GMM-to-GMM transport tasks.
We use the sliced Wasserstein distance (SWD) between the terminal and target distributions as our metric.
We use a standard SB model trained with the Diffusion Schr\"{o}dinger Bridge Matching (DSBM) algorithm~\citep{shi2023diffusion} as the baseline.
We test each trained model against three perturbation types, Gaussian noise (Noisy), directional shift (Shifted), and the DRSB worst-case distribution (Worst), to the initial distribution at test time.

\begin{wraptable}{r}{0.56\textwidth}
\centering
\caption{SWD ($\downarrow$) under input perturbations (mean $\pm$ std, 10 seeds). Org denotes the original input; the other conditions are defined in the text.}
\experimenttable
\begin{tabular}{@{}lcccc@{}}
\toprule
& \multicolumn{2}{c}{\textit{Gaussian-to-Gaussian}} & \multicolumn{2}{c}{\textit{GMM-to-GMM}} \\
\cmidrule(lr){2-3} \cmidrule(lr){4-5}
Input & DSBM & DRSB & DSBM & DRSB \\
\midrule
Org & \meanstd{\textbf{0.038}}{0.007} & \meanstd{0.123}{0.012} & \meanstd{\textbf{0.902}}{0.242} & \meanstd{1.165}{0.138} \\
Noisy & \meanstd{0.684}{0.013} & \meanstd{\textbf{0.638}}{0.016} & \meanstd{2.407}{0.243} & \meanstd{\textbf{2.212}}{0.413} \\
Shifted & \meanstd{0.666}{0.011} & \meanstd{\textbf{0.501}}{0.014} & \meanstd{5.917}{0.127} & \meanstd{\textbf{5.840}}{0.095} \\
Worst & \meanstd{0.951}{0.013} & \meanstd{\textbf{0.710}}{0.013} & \meanstd{1.988}{0.100} & \meanstd{\textbf{1.478}}{0.100} \\
\bottomrule
\end{tabular}
\label{tab:2d_robustness}
\end{wraptable}

The results are summarized in \cref{tab:2d_robustness}.
On the original distribution, DSBM achieves lower SWD in both tasks.
This is because DRSB optimizes for worst-case performance over the ambiguity set, which allocates transport capacity across a wider region of the input space at the cost of nominal accuracy.
Under all three perturbation types, DRSB outperforms DSBM on both Gaussian and GMM tasks.
DSBM degrades under perturbations because the learned transport map is fitted to the training distribution, and inputs that deviate from the nominal support produce suboptimal trajectories.
The gap between DRSB and DSBM is largest on the worst-case distribution trained on DRSB.

\WFclear
\FloatBarrier
\Needspace{0.20\textheight}

\begingroup
\setlength{\intextsep}{0pt}
\captionsetup{justification=raggedright}
\begin{figure*}[!htbp]
\centering
\begin{minipage}[t]{0.54\textwidth}
\vspace{0pt}
\captionsetup{type=table}
\caption{SWD ($\downarrow$), mean $\pm$ std over 10 seeds. W-wst and S-wst are the worst-case inputs of DRSB-W and DRSB-S.}
\label{tab:wdrsb_vs_sdrsb}
{\experimenttable
\setlength{\tabcolsep}{1.5pt}
\begin{tabular}{@{}lcccc@{}}
\toprule
& \multicolumn{2}{c}{\textit{Gaussian-to-Gaussian}} & \multicolumn{2}{c}{\textit{GMM-to-GMM}} \\
\cmidrule(lr){2-3} \cmidrule(lr){4-5}
Input & DRSB-W & DRSB-S & DRSB-W & DRSB-S \\
\midrule
Org & \meanstd{0.144}{0.007} & \meanstd{\textbf{0.123}}{0.012} & \meanstd{\textbf{0.885}}{0.226} & \meanstd{1.165}{0.138} \\
Noisy & \meanstd{0.678}{0.013} & \meanstd{\textbf{0.638}}{0.016} & \meanstd{2.315}{0.236} & \meanstd{\textbf{2.212}}{0.413} \\
Shifted & \meanstd{0.517}{0.014} & \meanstd{\textbf{0.501}}{0.014} & \meanstd{5.967}{0.153} & \meanstd{\textbf{5.840}}{0.095} \\
W-wst & \meanstd{0.972}{0.007} & \meanstd{\textbf{0.881}}{0.018} & \meanstd{\textbf{1.503}}{0.052} & \meanstd{1.791}{0.102} \\
S-wst & \meanstd{0.812}{0.012} & \meanstd{\textbf{0.710}}{0.013} & \meanstd{1.946}{0.100} & \meanstd{\textbf{1.478}}{0.100} \\
\bottomrule
\end{tabular}
\par}
\end{minipage}\hfill
\begin{minipage}[t]{0.42\textwidth}
\vspace{0pt}
\centering
\includegraphics[width=0.48\linewidth]{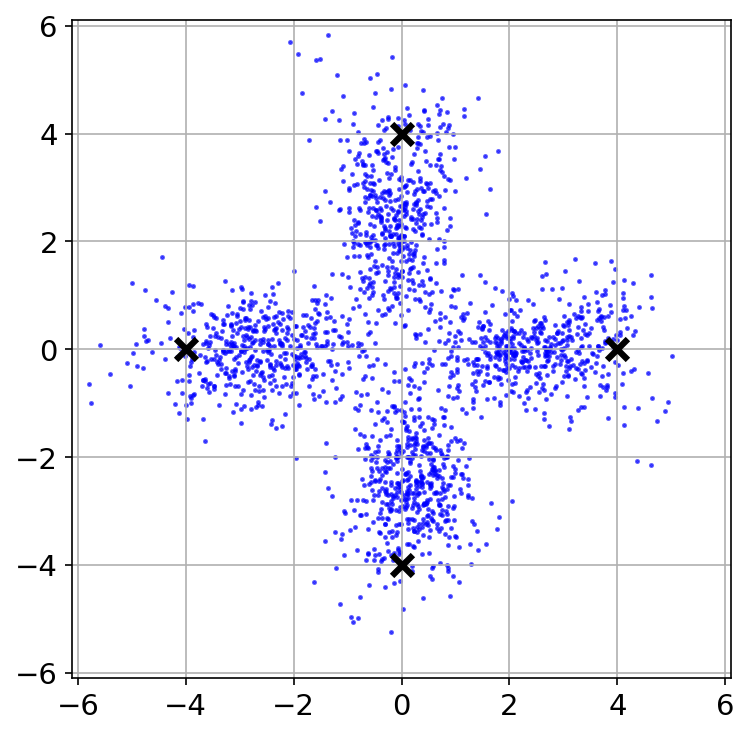}\hfill
\includegraphics[width=0.48\linewidth]{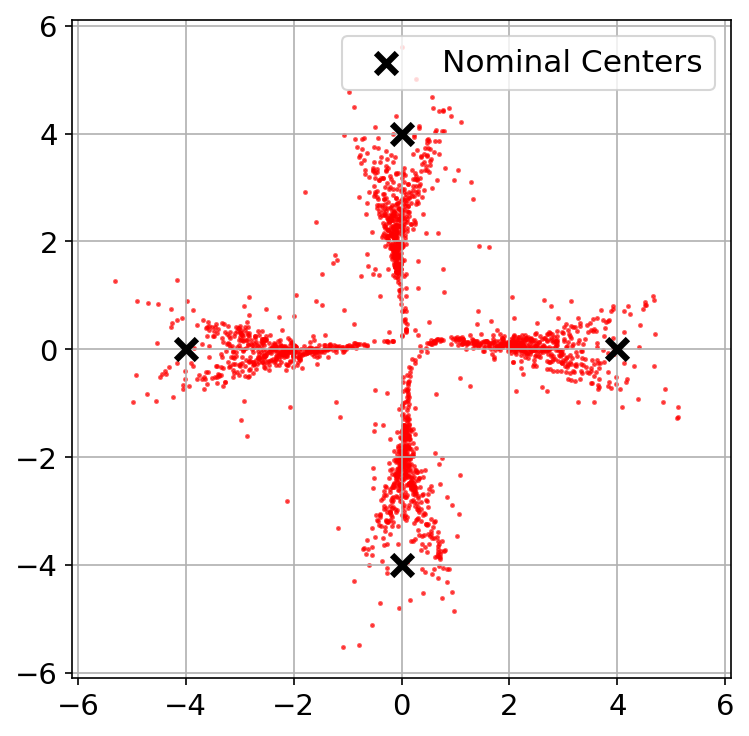}
\captionsetup{type=figure,labelfont={color=black},textfont={color=black}}
\caption{GMM adversarial samples from DRSB-S (left) and DRSB-W (right). Black crosses mark nominal centers.}
\label{fig:worst_case_comparison}
\end{minipage}
\end{figure*}

\subsection{Comparison between Wasserstein and Sinkhorn DRSB}
 
We compare the two DRSB variants across perturbation types (\cref{tab:wdrsb_vs_sdrsb,fig:worst_case_comparison}).
DRSB-S achieves lower mean SWD in eight of the ten reported conditions. These results indicate greater robustness of DRSB-S under most of the tested perturbation conditions.
\Cref{fig:worst_case_comparison} shows the computed worst case samples for the GMM task. DRSB-S samples are more dispersed due to the entropy regularization, whereas DRSB-W samples are more concentrated. One possible interpretation is that the broader spread exposes the DRSB-S controller to a wider range of inputs during training, which may help explain its performance under the tested perturbations.\par

\endgroup

\FloatBarrier
\subsection{Noise Augmentation and Generalization to Unseen Noise Levels}
\label{sec:rebuttal_noise}

We compare DRSB-S with a noise-augmented DSBM baseline on GMM-to-GMM transport to determine whether DRSB provides benefits beyond noise augmentation. Each augmentation baseline uses Gaussian noise at $\delta_{\mathrm{train}}=1$ or $2$, with the number of perturbed samples equal to 20\% of the clean dataset size. This makes the number of training steps involving perturbed samples comparable to that in the additional DRSB stage. We evaluate each model at test-time noise levels $\delta\in\{1,2,3,4\}$.

\begin{wraptable}{r}{0.53\textwidth}
\centering
\captionsetup{labelfont={color=black},textfont={color=black}}
\caption{GMM-to-GMM SWD ($\downarrow$) with noise augmentation (mean $\pm$ std, 10 seeds). Test levels $\delta=3,4$ are unseen by both augmentation baselines.}
\label{tab:rebuttal_noise}
\experimenttable
\begin{tabular}{@{}lcccc@{}}
\toprule
& DSBM & \multicolumn{2}{c}{DSBM + Aug} & DRSB-S \\
\cmidrule(lr){3-4}
$\delta$ & & $\delta_{\mathrm{train}}=1$ & $\delta_{\mathrm{train}}=2$ & \\
\midrule
Org & \meanstd{0.902}{0.242} & \meanstd{\textbf{0.850}}{0.178} & \meanstd{0.907}{0.153} & \meanstd{1.165}{0.138} \\
$1.0$ & \meanstd{1.287}{0.112} & \meanstd{\textbf{1.042}}{0.246} & \meanstd{1.119}{0.196} & \meanstd{1.150}{0.206} \\
$2.0$ & \meanstd{1.842}{0.086} & \meanstd{1.881}{0.145} & \meanstd{1.800}{0.153} & \meanstd{\textbf{1.636}}{0.157} \\
$3.0$ & \meanstd{2.175}{0.118} & \meanstd{2.298}{0.146} & \meanstd{2.220}{0.165} & \meanstd{\textbf{1.953}}{0.095} \\
$4.0$ & \meanstd{2.407}{0.243} & \meanstd{2.425}{0.129} & \meanstd{2.308}{0.138} & \meanstd{\textbf{2.212}}{0.413} \\
\bottomrule
\end{tabular}
\end{wraptable}

\Cref{tab:rebuttal_noise} shows that augmentation is effective at the noise level used in training, but its benefit does not transfer consistently to unseen levels. Augmentation at \mbox{$\delta_{\mathrm{train}}=1$} gives the lowest mean SWD at $\delta=1$ (1.042), yet performs worse than unaugmented DSBM at every tested $\delta\!\geq\!2$. Augmentation at \mbox{$\delta_{\mathrm{train}}=2$} improves performance at the matched level and at some other levels, but is worse than DSBM at the unseen level $\delta=3$.
DRSB-S achieves lower mean SWD than both augmentation baselines at every tested $\delta\!\geq\!2$. DRSB trains against an adaptive adversarial distribution within an ambiguity set, without specifying a test noise level. The results support broader generalization across the tested noise magnitudes than fixed-level augmentation. This robustness comes at the cost of higher Org input SWD (1.165, compared with 0.850 and 0.907 for augmentation).

\par\WFclear

\FloatBarrier
\Needspace{0.20\textheight}
\begin{minipage}{\textwidth}
\subsection{Image-to-Image Translation under Noise}

\begingroup
\setlength{\intextsep}{0pt}
\captionsetup{justification=raggedright}
\noindent\begin{minipage}[t]{0.55\textwidth}
\vspace{0pt}
\centering
\captionsetup{type=table}
\caption{FFHQ FID ($\downarrow$) under Gaussian noise $\delta$ (mean $\pm$ std, 3 seeds).}
\label{tab:ffhq_fid}
{\experimenttable
\setlength{\tabcolsep}{1.5pt}
\begin{tabular}{@{}lcccc@{}}
\toprule
& \multicolumn{2}{c}{\textit{Woman $\rightarrow$ Man}} & \multicolumn{2}{c}{\textit{Man $\rightarrow$ Woman}} \\
\cmidrule(lr){2-3} \cmidrule(lr){4-5}
Noise level & DSBM & DRSB & DSBM & DRSB \\
\midrule
Org & \meanstd{\textbf{25.55}}{0.50} & \meanstd{26.52}{0.50} & \meanstd{\textbf{29.32}}{0.44} & \meanstd{29.57}{0.48} \\
$\delta=0.25$ & \meanstd{26.59}{0.48} & \meanstd{\textbf{25.82}}{0.49} & \meanstd{30.02}{0.49} & \meanstd{\textbf{29.69}}{0.42} \\
$\delta=0.5$ & \meanstd{26.15}{0.49} & \meanstd{\textbf{25.98}}{0.49} & \meanstd{30.46}{0.43} & \meanstd{\textbf{29.70}}{0.51} \\
$\delta=1.0$ & \meanstd{26.81}{0.52} & \meanstd{\textbf{25.94}}{0.47} & \meanstd{29.63}{0.50} & \meanstd{\textbf{29.48}}{0.51} \\
\bottomrule
\end{tabular}
\par}
\vspace{5pt}
\includegraphics[width=0.32\linewidth]{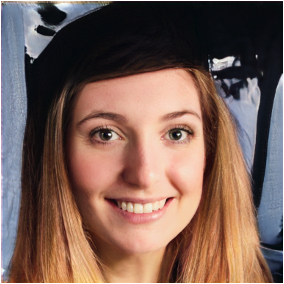}\hfill
\includegraphics[width=0.32\linewidth]{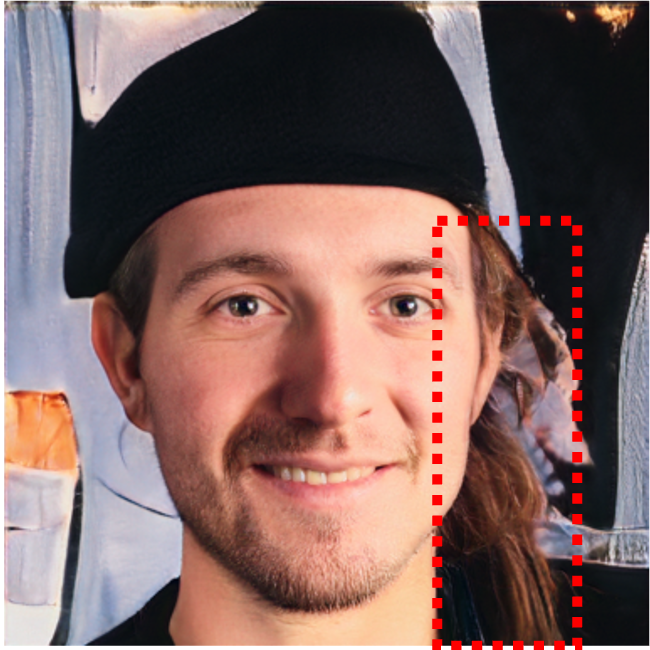}\hfill
\includegraphics[width=0.32\linewidth]{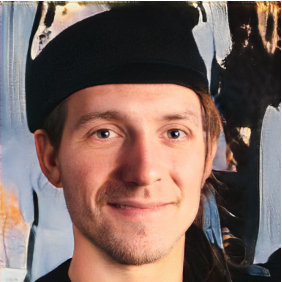}
\captionsetup{type=figure}
\caption{Woman-to-Man translation. Left to right: \mbox{perturbed input}, DSBM, and DRSB.}
\label{fig:ffhq_qualitative}
\end{minipage}\hfill
\begin{minipage}[t]{0.42\textwidth}
\vspace{0pt}

\leavevmode
We apply DRSB to image-to-image translation on the FFHQ dataset~\citep{karras2019style}, which contains images at a resolution of 1024$\times$1024 (CC BY-NC-SA 4.0).
Following the experimental settings in \cite{theodoropoulos2024feedback} and \cite{gushchin2024light}, we transport images in the 512-dimensional latent space of the ALAE autoencoder~\citep{pidhorskyi2020adversarial}. 
We use FID to evaluate generation quality after decoding the latent vector.
The results are shown in \cref{tab:ffhq_fid}.
DRSB outperforms DSBM across all $L_2$-norm noise levels. DSBM achieves a lower FID score than DRSB on nominal inputs.
\cref{fig:ffhq_qualitative} shows a qualitative example. Even when the effect of the added latent-space noise is visually imperceptible in the decoded input image, \mbox{DSBM generates visible artifacts}, whereas DRSB produces a clean translation.

\end{minipage}
\endgroup
\end{minipage}

\FloatBarrier
\Needspace{4\baselineskip}
\subsection{Additional Experimental Analysis}

\textbf{Computation cost.} On a single NVIDIA RTX 4090, the additional DRSB stage takes 404\,s on Gaussian transport and 1.4\,h on FFHQ, after DSBM pretraining of 754\,s and 8.5\,h. Total training costs are therefore approximately $1.54\times$ and $1.16\times$ DSBM. Network gradient updates dominate runtime, accounting for about 80\% on Gaussian transport and over 90\% on FFHQ. Worst-case optimization accounts for less than 1.5\% in the measured runs. DRSB-W uses up to 90 inner gradient steps, whereas DRSB-S samples a Gaussian adversary using the closed-form dual variable of the linearized problem. Appendix~\ref{sec:rebuttal_compute} reports the training configurations and stage-wise costs.

\textbf{Empirical convergence.} We track the variational objective and update diagnostics over 50 epochs on 2D Gaussian transport with $\rho=10$ and $\epsilon=12$. The objective stabilizes near 312, with a final value of 311.37. At all reported epochs $k\geq25$, the normalized ambiguity-constraint residual is below 0.005. The final dual variable is 3.483, close to the value 3.496 predicted by the active-constraint SOC relation. The terminal critic mean and mean squared control magnitude also stabilize, ending at 1.550 and 24.30. A separate 120-iteration run at constant learning rate $10^{-7}$ also settles near 312, supporting stabilization beyond the decaying learning-rate schedule. These observations do not establish stationarity of all three blocks or convergence to a saddle point (Appendix~\ref{sec:rebuttal_convergence}).

\textbf{SOC optimality condition.} The adversarial updates use the relation $\nabla_{X_0}V(u,X_0;g)=-u_0(X_0)$ from \cref{eq:optimal_control}. To assess its accuracy for the trained controller, we compute the cost gradient by discrete adjoint recursion and compare it with $-u^\theta(0,X_0)$ on the same Gaussian task. Across the five evaluated epochs, their direction cosines are 0.9996--0.9998 and gradient-to-control magnitude ratios are 0.9310--1.0761. At epoch 50, these values are 0.9998 and 1.0089. Thus, both direction and magnitude approximately satisfy the SOC relation on the evaluated samples, supporting its use in the adversarial update for this experiment. This is an empirical consistency check, not a proof of exact SOC optimality (Appendix~\ref{sec:rebuttal_soc}).

\FloatBarrier
\endgroup

\Needspace{23\baselineskip}
\section{Related Work}

\noindent \textbf{Distributionally Robust Optimization.}
Beginning with the work of~\citet{scarf1957min}, early DRO focused on moment-based ambiguity sets~\citep{ghaoui2003worst,delage2010distributionally}. Because moments have limited expressive power, the DRO literature shifted toward distance-based ambiguity sets.
Wasserstein DRO~\citep{mohajerin2018data,zhao2018data,blanchet2019quantifying,gao2023distributionally,zhang2025short} has attracted attention because of the geometric interpretation of the Wasserstein distance.
KL divergence offers convexity and a tractable inner maximization~\citep{ben2013robust,hu2013kullback}. However, its absolute continuity requirement and abstract interpretation are limitations.
Recently, Sinkhorn DRO has been proposed to improve computational tractability while modeling realistic uncertainty~\citep{wang2025sinkhorn}.

\noindent \textbf{Robust Optimal Transport.}
The OT literature studies how to reduce the sensitivity of transport maps to outliers.
One approach relaxes the marginal constraints in OT~\citep{mukherjee2021outlier}. In particular, unbalanced OT (UOT) formulations are used for this purpose~\citep{balaji2020robust,fatras2021unbalanced,choi2023generative}.
More closely related to our approach, the soft-constrained Schr\"{o}dinger Bridge~\citep{garg2024soft} relaxes terminal boundary constraints to prevent overfitting to noisy training samples. \citet{wang2024outlier} use a UOT distance within a DRO framework to handle outliers in training data.
These methods and DRSB address different forms of robustness. DRSB focuses on distribution shifts at test time.

\section{Conclusion and Limitations}

We formulate and solve the DRSB problem, which aims to transport noisy or shifted input distributions to the target distribution. We focus on the Wasserstein and Sinkhorn variants of DRSB. A globally solved Wasserstein inner problem bounds the fixed-terminal-cost objective, while the linearized Sinkhorn inner problem permits Gaussian sampling.
Exploring other distances for DRSB is a direction for future work.
Minimax adversarial training can be unstable in high-dimensional settings and requires careful hyperparameter tuning and initialization. The additional experiments quantify noise robustness and training overhead. The convergence diagnostics are empirical and do not establish convergence of the full algorithm.

\clearpage
\appendix

\section{Proofs of Results in Section \ref{sec:drsb}}
\label{appendix:drsb_proofs}

\subsection{Proof of Proposition~\ref{prop:g_guidance}}

Fix an admissible pair $(u,\mu)$ and abbreviate $p^{u,\mu}_1$ as $p^u_1$ in this subsection. Assume that $p^u_1$ and $\nu$ have positive densities on a common support and finite KL divergence. The function class is unrestricted, and expectations are population expectations.

\begin{lemma}[Variational form of KL]
\label{lem:variational_kl}
Let $\Phi(g) := \mathbb{E}_{p^u_1}[g(X)] - \log \mathbb{E}_{\nu}[e^{g(X)}]$ be defined over measurable functions $g: \mathcal{X} \rightarrow \mathbb{R}$ with $\mathbb{E}_{p^u_1}[|g|]<\infty$ and $0<\mathbb{E}_{\nu}[e^g]<\infty$. Then, for every such $g$,
\begin{equation}
    \Phi(g) \leq \mathrm{KL}(p^u_1\| \nu)
\end{equation}
and equality holds when
\begin{equation}
    g^*(x) := \log \left(\frac{p^u_1(x)}{\nu(x)} \right).
\end{equation}
\end{lemma}
\begin{proof}
Introduce a probability measure $Q_g$ such that
\begin{equation}
    \frac{dQ_g}{d\nu} = \frac{e^g}{\mathbb{E}_\nu[e^g]}
\end{equation}
Then $g$ can be written as
\begin{equation}
    g(X) = \log \frac{dQ_g}{d\nu} + \log \mathbb{E}_{\nu}[e^{g(X)}]
\end{equation}
and $\Phi(g)$ becomes
\begin{align}
    \Phi(g) &= \int g(X) \, dp^u_1(X) - \log \mathbb{E}_{\nu}[e^{g(X)}] \\
    &= \int \left\{ \log \frac{dQ_g}{d\nu} + \log \mathbb{E}_{\nu}[e^{g(X)}] \right\} dp^u_1(X) - \log \mathbb{E}_{\nu}[e^{g(X)}] \\
    &= \int \log \frac{dQ_g}{d\nu} \, dp^u_1 \\
    &= \int \left(\log \frac{dQ_g}{dp^u_1} + \log \frac{dp^u_1}{d\nu}\right) \, dp^u_1 \\
    &= \mathrm{KL}(p^u_1\| \nu)  - \mathrm{KL}(p^u_1\|Q_g) \leq \mathrm{KL}(p^u_1\| \nu)
\end{align}
Setting $g^*(X) := \log (dp^u_1/d\nu)$ gives $\log \mathbb{E}_{\nu}[e^{g^*}] = \log \int (dp^u_1/d\nu) \, d\nu = 0$, hence
\begin{equation}
    \Phi(g^*) = \int \log \left( \frac{dp^u_1}{d\nu} \right) dp^u_1 = \mathrm{KL}(p^u_1\|\nu).
\end{equation}
\end{proof}
Therefore, the optimal $g$ recovers the log-density ratio $\log(p^u_1/\nu)$.

\begin{lemma}[Binary Cross-Entropy Minimizer]
\label{lem:bce_minimizer}
The minimizer of the binary cross-entropy loss
\begin{equation}
    \mathcal{L}(g) = \mathbb{E}_{x \sim p^u_1}\left[\log \left(1+e^{-g(x)}\right)\right] + \mathbb{E}_{x\sim \nu}  \left[ \log \left(1+e^{g(x)} \right) \right]
\end{equation}
is $g^*(x) = \log \left( p^u_1(x) / \nu(x) \right)$.
\end{lemma}
\begin{proof}
For each $x$, minimize the population-loss integrand over the scalar $z=g(x)$:
\begin{align}
    \ell_x(z) &= p^u_1(x)\log\left(1+e^{-z}\right)+\nu(x)\log\left(1+e^z\right), \\
    \ell_x'(z) &= \frac{-p^u_1(x)+e^z\nu(x)}{1+e^z}, \\
    \ell_x''(z) &= \frac{\left(p^u_1(x)+\nu(x)\right)e^z}{\left(1+e^z\right)^2}>0.
\end{align}
Thus the pointwise minimum satisfies $e^{g^*(x)}\nu(x)=p^u_1(x)$, giving
\begin{equation}
    g^*(x)=\log\frac{p^u_1(x)}{\nu(x)}.
\end{equation}
Integrating the pointwise inequality proves the result over unrestricted measurable functions. A stationary point in neural-network parameters need not satisfy this pointwise condition.
\end{proof}

Proposition~(\ref{prop:g_guidance}) follows from Lemma~(\ref{lem:variational_kl}) and Lemma~(\ref{lem:bce_minimizer}).
By Lemma~(\ref{lem:bce_minimizer}), the minimizer of the binary cross-entropy loss recovers the log-density ratio $g^*(x) = \log(p^u_1(x)/\nu(x))$.
Substituting $g^*$ into the variational form in Lemma~(\ref{lem:variational_kl}) achieves $\Phi(g^*) = \mathrm{KL}(p^u_1 \| \nu)$, which gives \cref{eq:kl_variational}.

\subsection{\texorpdfstring{Proof of \cref{thm:DRSB_solved_with_DRO}}{Proof of the Variational Formulation}}

\begin{proof}
For every admissible pair $(u,\mu)$, Lemma~\ref{lem:variational_kl} gives
\begin{equation}
    \mathrm{KL}\left(p^{u,\mu}_1\|\nu\right)
    =\sup_g\left\{\mathbb{E}_{p^{u,\mu}_1}[g(X_1)]-\log\mathbb{E}_{\nu}[e^{g(X_1)}]\right\}.
\end{equation}
Substitution into \cref{eq:drsb_obj}, followed by conditioning on $X_0$, yields
\begin{align}
    &\min_u\max_{\mu\in\mathfrak{M}_{\epsilon}(\hat{\mu})}\sup_g
    \left\{\mathbb{E}_{p^{u,\mu}}\left[\int_0^1\frac{1}{2}\|u_t^\theta(X_t)\|^2dt+\rho g(X_1)\right]
    -\rho\log\mathbb{E}_{\nu}[e^{g(X_1)}]\right\} \\
    &=\min_u\max_{\mu\in\mathfrak{M}_{\epsilon}(\hat{\mu})}\sup_g
    \left\{\mathbb{E}_{X_0\sim\mu}[V(u,X_0;g)]
    -\rho\log\mathbb{E}_{\nu}[e^{g(X_1)}]\right\},
\end{align}
which is \cref{eq:drsb_variational}. This substitution is pointwise in $(u,\mu)$ and does not interchange the minimization over $u$ with an inner maximization.
For a fixed pair, the normalized maximizer is
\begin{equation}
    g^*(x;u,\mu)=\log\frac{p^{u,\mu}_1(x)}{\nu(x)},
    \qquad \log\mathbb{E}_{\nu}[e^{g^*(X_1;u,\mu)}]=0.
\end{equation}
Because this maximizer depends on both $u$ and $\mu$, it cannot be fixed while optimizing them. For a separately fixed $g$, the logarithmic term is constant with respect to $u$ and $\mu$ and may be omitted from their optimization. This gives the fixed-$g$ subproblem in \cref{eq:drsb_dro_form}, rather than an equivalent formulation with a single fixed $g^*$ for all pairs.
\end{proof}

\section{Details on the Derivation of DRSB Variants}
\label{appendix:detail_drsb_vars}

Throughout this section, $g$ is fixed and $V(u,x)$ abbreviates $V(u,x;g)$; in an inner distribution update, $u$ is also fixed. The dual identities below concern this fixed-cost subproblem, under the finiteness and regularity assumptions of the cited DRO duality results.

\subsection{Details on Wasserstein DRSB}
\label{appendix:wdrsb}
We start from the fixed-$g$ Wasserstein DRSB subproblem
\begin{subequations}
\begin{align}
    \min_{u} \max_{\mu \in \mathfrak{M}_{\epsilon}(\hat{\mu})} & \mathbb{E}_{p^u} \left[ \int_0^1  \frac{1}{2}\| u^{\theta}_t(X_t)\|^2 dt + \rho g(X_1) \right] \\
    \text{s.t.} \,\, & dX_t = u^{\theta}_t(X_t) dt + \sigma dW_t,\, X_0 \sim \mu \\
    \quad & \mathfrak{M}_{\epsilon}(\hat{\mu}) = \{ \mu : W_2^2(\mu, \hat{\mu}) \le \epsilon \}
\end{align}
\end{subequations}
where the squared Wasserstein-2 distance is defined as
\begin{equation}
    W_2^2(\mu, \hat{\mu}) := \min_{\gamma \in \Gamma(\mu, \hat{\mu})} \int \|X_0 - \hat{X}_0\|^2 \gamma(dX_0, d\hat{X}_0).
\end{equation}
\label{appendix:prob:wdrsb_prob}

Consider the inner maximization
\begin{subequations}
\begin{align}
    \max_{\mu \in \mathfrak{M}_{\epsilon}(\hat{\mu})} & \mathbb{E}_{p^u} \left[ \int_0^1  \frac{1}{2}\| u^{\theta}_t(X_t)\|^2 dt + \rho g(X_1) \right] \\
    \text{s.t.} \,\, & dX_t = u^{\theta}_t(X_t) dt + \sigma dW_t,\, X_0 \sim \mu \\
    \quad & \mathfrak{M}_{\epsilon}(\hat{\mu}) = \{ \mu : W_2^2(\mu, \hat{\mu}) \le \epsilon \}
\end{align}
\end{subequations}
Define the Lagrangian as
\begin{align}
\mathcal{L}(\mu, \lambda) &:= \mathbb{E}_{\mathbf{X} \sim p^u}\left[ \int_0^1 \frac{1}{2}\| u^{\theta}_t(X_t)\|^2 dt + \rho g(X_1)\right] - \lambda \left( \int \|X_0 - \hat{X}_0\|^2 \pi(dX_0, d\hat{X}_0) - \varepsilon \right)
\end{align}
Conditioning on $X_0$,
\begin{align}
\mathcal{L}(\mu, \lambda) &= \mathbb{E}_{X_0 \sim \mu} \left[ V(u, X_0) \right] - \lambda \left( \int \|X_0 - \hat{X}_0\|^2 \pi(dX_0, d\hat{X}_0) - \varepsilon \right)
\end{align}
where $V(u, X_0) := \mathbb{E}_{\mathbf{X}_{\cdot|0} \sim p^u_{\cdot|0}}\left[ \int_0^1 \frac{1}{2}\| u^{\theta}_t(X_t)\|^2 dt + \rho g(X_1) \middle| X_0\right]$ is the cost function at $t=0$.
Then we have
\begin{align}
\mathcal{L}(\mu, \lambda) &= \lambda \varepsilon + \int V(u, X_0) \pi(X_0, \hat{X}_0) \, dX_0 \, d\hat{X}_0 - \lambda \int \|X_0 - \hat{X}_0\|^2 \pi(X_0, \hat{X}_0) \, dX_0 \, d\hat{X}_0
\end{align}
The dual problem is
\begin{align}
& \min_{\lambda \geq 0} \, \lambda \varepsilon + \max_{\mu} \, \mathbb{E}_{(X_0,\hat{X}_0)\sim \pi} \left[ V(X_0) - \lambda \|X_0 - \hat{X}_0\|^2 \right]
\end{align}
By the interchangeability principle~\citep{zhang2025short}, the dual problem can be written as
\begin{align}
    & \min_{\lambda \geq 0} \, \lambda \varepsilon +  \int \max_{X_0} \left\{ V(X_0) - \lambda \|X_0 - \hat{X}_0\|^2 \right\} \hat{\mu}(\hat{X}_0) \, d\hat{X}_0 \\
    &= \min_{\lambda \geq 0} \, \lambda \varepsilon -  \int \min_{X_0} \left\{ \lambda \|X_0 - \hat{X}_0\|^2 - V(X_0) \right\} \hat{\mu}(\hat{X}_0) \, d\hat{X}_0
\end{align}
For the empirical distribution $\hat{\mu} = \frac{1}{N}\sum_{i=1}^N \delta_{\hat{X}_0^i}$, the expectation reduces to a finite sum, giving
\begin{equation}
    \min_{\lambda \geq 0} \, \lambda \varepsilon - \frac{1}{N}\sum_{i=1}^N \min_{X_0} \left\{ \lambda \|X_0 - \hat{X}^i_0\|^2 - V(X_0) \right\}.
\end{equation}

\subsubsection{\texorpdfstring{Derivation of Update Rules (\cref{eq:worst_case_gd} and \cref{eq:lambda_update_rule})}{Derivation of Update Rules}}

Equivalently, the samplewise maximization objective is
\begin{equation}
    \max_{X_0} \mathcal{L}_1(X_0) :=  V(X_0) -\lambda \|X_0 -\hat{X}_0^i\|^2.
\label{appendix:eq:one_wst_case}
\end{equation}
Taking the gradient of $\mathcal{L}_1(X_0)$ with respect to $X_0$ gives
\begin{equation}
    \nabla_{x} \mathcal{L}_1(X_0) = \nabla_{x} V(X_0) - 2\lambda (X_0 - \hat{X}_0^i)
\end{equation}
If $u=u^*$ solves the SOC problem for this fixed $g$ and its value function is differentiable, \cref{eq:optimal_control} gives
\begin{equation}
    \nabla_{x} \mathcal{L}_1(X_0) = -u^*_0(X_0) - 2\lambda (X_0 - \hat{X}_0^i)
\end{equation}
Replacing $u^*$ by the trained drift gives the implemented ascent update. This substitution is approximate unless the fixed-$g$ SOC optimality condition is satisfied:
\begin{equation*}
    X_0 \leftarrow X_0 + \eta \left[ - u^\theta_0(X_0) - 2\lambda (X_0 - \hat{X}_0^i) \right]
\end{equation*}
where $\eta$ is the gradient ascent step size.
Each sample update in \cref{eq:worst_case_gd} requires one evaluation of $u^\theta_0$; total overhead also depends on the number of sample and dual updates.

Similarly, we can derive the update rule in \cref{eq:lambda_update_rule} from the optimality condition for the worst-case optimization.
For \cref{appendix:eq:one_wst_case}, the first-order optimality condition is
\begin{equation}
    \nabla_{x} \mathcal{L}_1(X^{(i)*}_0) = \nabla_{x} V(X^{(i)*}_0) - 2\lambda (X^{(i)*}_0 - \hat{X}_0^i) = 0
\end{equation}
Therefore, the worst-case sample satisfies
\begin{equation}
    \nabla_{x} V(X^{(i)*}_0) = 2\lambda (X^{(i)*}_0 - \hat{X}_0^i)
\label{appendix:eq:wst_case_optimality}
\end{equation}
Define the objective for the dual variable as
\begin{equation}
    \inf_{\lambda \geq 0} f(\lambda) := \lambda \varepsilon -  \frac{1}{N}\sum_{i=1}^N \left[ \lambda \|X^{(i)*}_0 - \hat{X}^i_0\|^2 - V(X^{(i)*}_0) \right]
\end{equation}
By taking the gradient with respect to $\lambda$, we have
\begin{align}
    \frac{df(\lambda)}{d\lambda} &= \varepsilon + \frac{1}{N}\sum_{i=1}^N \left[  \nabla_{x} V(X^{(i)*}_0) \frac{dX_0^{(i)*}}{d\lambda} - \|X^{(i)*}_0 - \hat{X}^i_0\|^2 \right. \notag \\
    & \qquad \left. - 2\lambda(X_0^{(i)*} - \hat{X}_0^i) \frac{dX_0^{(i)*}}{d\lambda} \right]
\end{align}
From \cref{appendix:eq:wst_case_optimality}, we can cancel out the $\frac{dX_0^{(i)*}}{d\lambda}$ terms:
\begin{equation}
    \frac{df(\lambda)}{d\lambda} = \varepsilon - \frac{1}{N}\sum_{i=1}^N \   \|X^{(i)*}_0 - \hat{X}^i_0\|^2
\end{equation}
Consequently, we can perform gradient descent on $\lambda$:
\begin{equation*}
    \lambda \leftarrow \lambda - \zeta \left( \varepsilon - \frac{1}{N}\sum_{i=1}^N \|X^{(i)*}_0 - \hat{X}^i_0\|^2 \right)
\end{equation*}
where $\zeta$ is the gradient descent step size. This update is valid only while $\lambda$ remains nonnegative; otherwise feasibility requires projection onto $[0,\infty)$.

\subsection{Details on Sinkhorn DRSB}
\label{appendix:SDRSB}
We start from the fixed-$g$ Sinkhorn DRSB subproblem
\begin{subequations}
\begin{align}
    \min_{u} \max_{\mu \in \mathfrak{M}^S_{\epsilon}(\hat{\mu})} &\quad \mathbb{E}_{p^u}\left[ \int_0^1 \frac{1}{2}\| u^{\theta}_t(X_t)\|^2 dt + \rho g(X_1)\right] \\
    \text{subject to} \quad & dX_t = u^{\theta}_t(X_t) dt + \sigma dW_t, \quad X_0 \sim \mu \\
    \quad & \mathfrak{M}^S_{\epsilon}(\hat{\mu}) = \{ \mu : \mathcal{W}_\kappa(\hat{\mu}, \mu) \le \epsilon \}
\end{align}
\end{subequations}
\label{appendix:prob:sinkhorn_DRSB}
where $\mathcal{W}_\kappa(\hat{\mu}, \mu)$ is the Sinkhorn distance with entropy regularization parameter $\kappa$:
\begin{equation}
    \mathcal{W}_{\kappa}(\hat{\mu}, \mu) = \min_{\gamma \in \Gamma(\hat{\mu}, \mu)} \left\{ \mathbb{E}_{(\hat{X}_0, X_0) \sim \gamma} \left[\|\hat{X}_0- X_0\|^2  +  \kappa \log \left( \frac{d\gamma(\hat{X}_0, X_0)}{d\hat{\mu}(\hat{X}_0) dX_0} \right)  \right] \right\}.
\end{equation}
Denote the value of the inner maximization by $\mathcal{V}$.
\begin{align}
    \mathcal{V} &:= \max_{\mu \in \mathfrak{M}^S_{\epsilon}(\hat{\mu})} \mathbb{E}_{X_0 \sim \mu}\left[V(X_0) \right]  \\
    &= \max_{\gamma} \min_{\lambda \geq 0} \lambda \varepsilon + \mathbb{E}_{\hat{X}_0\sim \hat{\mu}}\mathbb{E}_{X_0\sim \gamma(X_0|\hat{X}_0)} \notag \\
    & \qquad \left[V(X_0) - \lambda  \|\hat{X}_0 - X_0\|^2 - \lambda \kappa \log \left( \frac{d\gamma(X_0|\hat{X}_0)}{dX_0} \right)  \right]
\end{align}
Under the assumptions for strong duality in Sinkhorn DRO~\citep{wang2025sinkhorn}, we switch the order:
\begin{align}
    \mathcal{V} &= \min_{\lambda \geq 0} \lambda \varepsilon + \mathbb{E}_{\hat{X}_0\sim \hat{\mu}} \Bigg[ \max_{\gamma(X_0|\hat{X}_0)} \mathbb{E}_{X_0\sim \gamma(X_0|\hat{X}_0)} \notag \\
    & \qquad \left[V(X_0) - \lambda  \|\hat{X}_0 - X_0\|^2 - \lambda \kappa \log \left( \frac{d\gamma(X_0|\hat{X}_0)}{dX_0} \right)  \right] \Bigg]
\label{eq:derivation_inner_maximization}
\end{align}
For $\lambda>0$, $\kappa>0$, and a finite normalizing integral, the entropy-regularized inner maximization admits the following closed-form solution. Boundary cases at $\lambda=0$ are understood by limits.
The dual form is
\begin{equation}
    \mathcal{V} = \inf_{\lambda > 0} \lambda \varepsilon + \mathbb{E}_{\hat{X}_0\sim \hat{\mu}} \left[ \lambda \kappa \log \int e^{ \left( V(X_0) - \lambda \|X_0 - \hat{X}_0\|^2 \right)/\lambda \kappa} dX_0 \right]
\label{eq:sinkhorn_dual}
\end{equation}
and the optimal conditional distribution is
\begin{equation}
    \gamma^*(X_0 | \hat{X}_0) \propto e^{\left(V(X_0) - \lambda \|X_0 - \hat{X}_0\|^2\right)/\lambda\kappa}
\end{equation}

\subsubsection{Second-order approximation}
We approximate $V(X_0)$ up to second order:
\begin{equation}
    V(X_0) \approx V(\hat{X}_0) + \nabla V(\hat{X}_0)^\top \delta X_0 + \frac{1}{2} \delta X_0^\top \nabla_{xx}V(\hat{X}_0) \, \delta X_0
\end{equation}
where $\delta X_0 = X_0 - \hat{X}_0$.
Define $g := \frac{\nabla V(\hat{X}_0)}{\lambda \kappa}$ and $A := \frac{2}{\kappa}I_d - \frac{\nabla_{xx}V(\hat{X}_0)}{\lambda \kappa}$. Here $g$ denotes a local gradient vector; the terminal-cost function remains fixed.
For the quadratic approximation and $A\succ0$, the corresponding conditional distribution is Gaussian. The equalities below apply to the approximated cost:
\begin{align}
    \gamma^*(X_0 | \hat{X}_0) &\propto \exp \left( g^\top \delta X_0 - \frac{1}{2} \delta X_0^\top A \, \delta X_0 \right) \\
    &\propto \exp \left( -\frac{1}{2} (\delta X_0 - A^{-1}g )^\top A (\delta X_0 - A^{-1}g ) \right)
\end{align}
As a result,
\begin{align}
    \gamma^*(X_0 | \hat{X}_0) &= \mathcal{N}( \hat{X}_0 + A^{-1}g, \, A^{-1})
\end{align}
with the full expression
\begin{align}
    \gamma^*(X_0 | \hat{X}_0) = \mathcal{N}\left( \hat{X}_0 + \left(  \frac{2}{\kappa}I - \frac{\nabla_{xx}V}{\lambda \kappa} \right)^{-1} \frac{\nabla V}{\lambda \kappa}, \, \left( \frac{2}{\kappa}I - \frac{\nabla_{xx}V}{\lambda \kappa}  \right)^{-1} \right)
\end{align}
where $V$ and its derivatives are evaluated at $\hat{X}_0$.

\subsubsection{Simplification of the dual form}
Denote $I := \int e^{ \left( V(X_0) - \lambda \|X_0 - \hat{X}_0\|^2 \right)/\lambda \kappa} dX_0$.
For the same quadratic approximation, completing the square in terms of $g$ and $A$ gives the following approximate integral and dual objective:
\begin{align}
    I &= e^{ \frac{V(\hat{X}_0)}{\lambda \kappa} } \int \exp \left\{ g^\top \delta X_0 - \frac{1}{2}\delta X_0^\top A\, \delta X_0 \right\} dX_0 \\
    &= e^{ \frac{V(\hat{X}_0)}{\lambda \kappa} + \frac{1}{2} g^\top A^{-1} g } \int \exp \left\{ -\frac{1}{2}(\delta X_0 - A^{-1}g)^\top A (\delta X_0 - A^{-1}g) \right\} dX_0
\end{align}
Evaluating the Gaussian integral under the condition $A \succ 0$ gives
\begin{equation}
    I = (2\pi)^{d/2} \det (A^{-1})^{1/2} \exp \left( \frac{V(\hat{X}_0)}{\lambda \kappa} + \frac{1}{2} g^\top A^{-1} g \right)
\end{equation}
Therefore,
\begin{align}
    \lambda \kappa \log I = V(\hat{X}_0) + \frac{\lambda \kappa}{2} g^\top A^{-1}g + \lambda \kappa \log\left( (2\pi)^{d/2} \det (A^{-1})^{1/2} \right)
\end{align}
and the dual objective becomes
\begin{align}
    \mathcal{V} = \inf_{\lambda > 0} \lambda \left( \varepsilon + \frac{\kappa d}{2} \log 2\pi \right) + \mathbb{E}_{\hat{X}_0 \sim \hat{\mu}} \left[ V(\hat{X}_0) + \frac{\lambda \kappa}{2} g^\top A^{-1} g + \frac{\lambda \kappa}{2} \log \det(A^{-1}) \right]
\end{align}
Solving for $\lambda$ in this form is complicated due to the dependence of $A^{-1}$ on $\lambda$.

\subsubsection{First-order approximation}
In high-dimensional cases, computing the Hessian $\nabla_{xx}V(\hat{X}_0)$ is expensive.
For the linearized cost, the Hessian term is omitted, giving $A = \frac{2}{\kappa}I_d$ and $A^{-1} = \frac{\kappa}{2}I_d$. The following expressions are exact for this linearized inner problem.
The conditional distribution simplifies to
\begin{equation}
    \gamma^*(X_0 | \hat{X}_0) = \mathcal{N} \left( \hat{X}_0 + \frac{1}{2\lambda} \nabla V(\hat{X}_0) \,, \frac{\kappa}{2} I \right)
\label{eq:worst_condit_dist_opt_first_order_approx}
\end{equation}
The $\lambda$ optimization reduces to
\begin{equation}
    \inf_{\lambda > 0} \lambda \left( \varepsilon + \frac{\kappa d}{2} \log \pi \kappa \right) + \mathbb{E}_{\hat{X}_0 \sim \hat{\mu}} \left[ V(\hat{X}_0 ) + \frac{\|\nabla V(\hat{X}_0) \|^2}{4 \lambda} \right]
\end{equation}
For $\varepsilon+\frac{\kappa d}{2}\log(\pi\kappa)>0$ and $\mathbb{E}_{\hat{\mu}}[\|\nabla V\|^2]>0$, this yields the positive closed-form solution
\begin{equation}
    \lambda^* = \sqrt{\left( \varepsilon + \frac{\kappa d}{2} \log \pi \kappa \right)^{-1} \mathbb{E}_{\hat{X}_0 \sim \hat{\mu}} \left[ \frac{\|\nabla V(\hat{X}_0) \|^2}{4} \right]}
\label{eq:lambda_opt_first_order_approx}
\end{equation}

\section{Stochastic Optimal Control Theory}
\label{appendix:soc}
Here we derive \cref{eq:optimal_control} using stochastic optimal control theory. We refer the reader to \citet{kappen2005path} for more details.
The terminal cost is fixed, and the HJB derivation assumes a sufficiently smooth value function and admissible controls for the verification argument. To apply this section to DRSB, replace the terminal cost $g$ below by $\rho g$. The identity holds for an optimal control and its optimal value, not for an arbitrary intermediate neural drift.
Consider the stochastic optimal control (SOC) problem
\begin{subequations}
\begin{align}
    \min_u \quad & \mathbb{E}\left[ \int_0^1 \frac{1}{2} \| u_t(X_t) \|^2 dt + g(X_1)\right] \\
    \text{subject to} \quad & dX_t = u_t(X_t) dt + \sigma dW_t \label{eq:sde} \\
    \quad & X_0 \sim \mu ,\,\,
\end{align}
\label{eq:soc_prob}
\end{subequations}
Define the cost function as
\begin{equation}
    c_t(X_t, u_t) := \mathbb{E}\left[ \int_t^1 \frac{1}{2} \| u_s(X_s) \|^2 ds + g(X_1) \right]
\end{equation}
The value function at time $t$ is the optimal cost from time $t$ to $1$.
\begin{equation}
    V_t(x) := \min_{\substack{u(s) \\ t \leq s \leq 1 }} \mathbb{E} \left[ \int_t^1 \frac{1}{2} \| u_s(X_s) \|^2 ds + g(X_1) \,\, \middle| \,\, X_t=x  \right]
\end{equation}
The cost function depends on the control $u_t$, whereas the value function is obtained by minimizing over the control. We split the integral into the intervals from $t$ to $t+dt$ and from $t+dt$ to 1. Substituting $t=1$ gives $V_1(x) = g(x)$.
\begin{equation}
     V_t(X_t) = \min_{\substack{u(s) \\ t \leq s \leq 1 }} \mathbb{E} \left[ \int_t^{t+dt} \frac{1}{2} \| u_s(X_s) \|^2 ds + \int_{t+dt}^{1} \frac{1}{2} \| u_s(X_s) \|^2 ds + g(X_1)  \,\, \middle| \,\, X_t   \right]
\end{equation}
Using Bellman's principle of optimality~\citep{bellman1954theory}, we have
\begin{align}
    V_t(X_t) &= \min_{\substack{u(s) \\ t \leq s \leq 1 }} \mathbb{E} \Bigg[ \int_t^{t+dt} \frac{1}{2} \| u_s(X_s) \|^2 ds \notag \\
    & \quad + \min_{\substack{u(s) \\ t+dt \leq s \leq 1}} \mathbb{E} \left[ \int_{t+dt}^{1} \frac{1}{2} \| u_s(X_s) \|^2 ds + g(X_1) \,\middle|\, X_{t+dt} \right]  \,\Bigg|\, X_t \Bigg] \\
    &= \min_{\substack{u(s) \\ t \leq s \leq 1 }} \mathbb{E} \left[ \int_t^{t+dt} \frac{1}{2} \| u_s(X_s) \|^2 ds + V_{t+dt}(X_{t+dt})  \,\middle|\, X_t   \right] \label{eq:value_fnc_exp}
\end{align}
Applying a Taylor expansion to $V_{t+dt}(X_{t+dt})$ at $(X_t, t)$ gives
\begin{equation}
    \mathbb{E}[ V_{t+dt}(X_{t+dt})] \approx \mathbb{E}\left[ V_t(X_t) + \frac{\partial V_t(X_t)}{\partial t} dt + (\nabla V_t)^\top dX_t + \frac{1}{2} (dX_t)^\top \nabla^2 V (dX_t) \right]
\end{equation}
Applying stochastic calculus to \cref{eq:sde} gives
\begin{align}
    \mathbb{E}[dX_t] &= u_t(X_t) dt \\
    \mathbb{E}[(dX_t)(dX_t)^\top] &= \sigma^2 \mathbf{I} dt
\end{align}
Therefore,
\begin{equation}
    \mathbb{E}[ V_{t+dt}(X_{t+dt})] = \mathbb{E}\left[ V_t(X_t) + \frac{\partial V_t(X_t)}{\partial t} dt + (\nabla V_t)^\top u_t(X_t) dt + \frac{\sigma^2}{2} \Delta V_t dt \right] \label{eq:taylor_exp}
\end{equation}
Taking the limit as $dt \rightarrow 0$ and substituting \cref{eq:taylor_exp} into \cref{eq:value_fnc_exp} gives
\begin{equation}
    - \frac{\partial V_t(X_t)}{\partial t} = \min_u \left\{   \frac{1}{2} \| u_t(X_t) \|^2 + (\nabla V_t)^\top u_t(X_t) + \frac{\sigma^2}{2} \Delta V_t  \right\}
\end{equation}
which is the Hamilton--Jacobi--Bellman (HJB) equation.
By differentiating the right-hand side, we obtain the optimal control $u^*$:
\begin{equation}
    u_t^*(X_t) = - \nabla V_t(X_t)
\label{appendix:eq:opt_control}
\end{equation}
At time $t=0$, we have \cref{eq:optimal_control}.

\section{Solving the SOC with Adjoint Matching}
\label{appendix:AM}

\subsection{Adjoint Matching}
This section summarizes results from the original adjoint matching paper. We refer the reader to \citet{domingo2024adjoint}, Sections 5.2 and E, for more details.
The initial distribution and terminal cost are fixed throughout this section. We assume sufficient differentiability and integrability for pathwise differentiation with additive noise and interchange with expectation. Gradients are column vectors and $\nabla_xu$ denotes the Jacobian.
Adjoint matching solves the optimal control problem
\begin{align}
    & \min_{u^\theta} \,\, \mathbb{E}_{p^u} \left[ \int_0^1 \frac{1}{2} \|u_t^\theta(X^u_t)\|^2 dt + g(X_1^u)\right] \\
    & \text{s.t.} \,\, dX^u_t = u_t^\theta(X^u_t)dt + \sigma_t dW_t \\
    & \,\,\quad X_0 \sim \mu
\end{align}
Here, we explicitly include the superscript $u$ in the state notation $X^u$, unlike in previous sections. This notation is useful in the derivation of the adjoint matching objective.

\subsubsection{Derivation of the Adjoint System}

We introduce a Lagrange multiplier process $a: \Omega \times[0,1] \rightarrow \mathbb{R}^n$ for the dynamics. The Lagrangian of the optimal control problem is
\begin{align}
    \mathcal{L}(u_t^\theta, a_t ; \bm{X}^u_t, t) &= \int_t^1 \frac{1}{2} \|u_t^\theta(X^u_t)\|^2 dt + g(X_1^u) \notag \\
    & \quad - \int_t^1 a_t^\top \left( dX^u_t - u_t^\theta(X_t^u)dt - \sigma_t dW_t \right) \\
    &=  \int_t^1 \frac{1}{2} \|u_t^\theta(X^u_t)\|^2 dt + g(X_1^u) - \int_t^1 a_t^\top  dX^u_t \notag \\
    & \quad + \int_t^1 a_t^\top u_t^\theta(X_t^u)dt + \int_t^1 \sigma_t a_t^\top dW_t
\end{align}
Applying integration by parts to the second integral gives
\begin{align}
    \mathcal{L}(u_t^\theta, a_t, \bm{X}^u_t) &= \int_t^1 \frac{1}{2} \|u_t^\theta(X^u_t)\|^2 dt + g(X_1^u) - a_1^\top X_1 + a_t^\top X_t \notag \\
    & \quad + \int_t^1 (X_t^u)^\top \left( \frac{da_t}{dt} \right) dt + \int_t^1 a_t^\top u_t^\theta(X_t^u)dt + \int_t^1 \sigma_t a_t^\top dW_t
\end{align}
Taking the first variation with respect to the trajectory $\bm{X}_t$ yields
\begin{align}
    \frac{\delta}{\delta \bm{X}^u_t}\mathcal{L} &= (\nabla_x g(X_1^u) - a_1)^\top h_1 \notag \\
    & \quad + \int_t^1 \left( \frac{da_t}{dt} + \left(\nabla_x u_t^\theta(X_t^u)\right)^\top a_t + \nabla_x\left( \frac{1}{2} \|u_t^\theta(X^u_t)\|^2\right) \right)^\top h_t \, dt = 0,
\end{align}
which must hold for all $h_t$. Thus, the coefficients of $h_t$ must be zero, and the adjoint system
\begin{subequations}
\label{eq:adjoint-system}
\begin{empheq}[left={\empheqlbrace}]{align}
\frac{d a_t}{dt} &= - \left(\nabla_x u_t^\theta(X_t^u)\right)^\top a_t - \nabla_x\left( \tfrac{1}{2} \|u_t^\theta(X_t^u)\|^2 \right) \label{eq:adjoint-ode} \\
a_1 &= \nabla_x g(X_1^u) \label{eq:adjoint-terminal}
\end{empheq}
\end{subequations}
is derived.

Next, we add the dynamics constraint to the cost-to-go function $\mathcal{J}(u, x_t)$, using the adjoint state as its multiplier.
\begin{align}
    \mathcal{J}(u, x_t) &:= \int_t^1 \frac{1}{2} \| u_t^\theta(X^u_t)\|^2 dt + g(X^u_1) \\
    &= \int_t^1 \frac{1}{2} \| u_t^\theta(X^u_t)\|^2 dt + g(X^u_1) - \int_t^1 a_t^\top \left( dX^u_t - u_t^\theta(X_t^u)dt - \sigma_t dW_t \right) \\
    &= \int_t^1 \frac{1}{2} \| u_t^\theta(X^u_t)\|^2 dt + g(X^u_1) - a_1^\top X^u_1 + a_t^\top X^u_t \notag \\
    & \quad + \int_t^1 (X_t^u)^\top \left( \frac{da_t}{dt} \right) dt + \int_t^1 a_t^\top u_t^\theta(X_t^u)dt + \int_t^1 \sigma_t a_t^\top dW_t
\end{align}
Taking the gradient with respect to the initial state $X_t=x_t$ yields
\begin{align}
    \nabla_{x_t}\mathcal{J}(u,x_t)
    &=\int_t^1\left(\nabla_{x_t}X_s^u\right)^\top
    \left(\frac{da_s}{ds}+\left(\nabla_xu_s^\theta(X_s^u)\right)^\top a_s
    +\nabla_x\left(\frac{1}{2}\|u_s^\theta(X_s^u)\|^2\right)\right)ds \notag\\
    &\quad+\left(\nabla_{x_t}X_1^u\right)^\top\left(\nabla_xg(X_1^u)-a_1\right)+a_t\\
    &=a_t\qquad\left(\text{by \cref{eq:adjoint-system}}\right).
\end{align}
As a result, the solution of the adjoint system in \cref{eq:adjoint-system} is
\begin{equation}
    a_t(\bm{X}^u_t, u^\theta) = \nabla_{x_t}\left(\int_t^1 \frac{1}{2} \| u_s^\theta(X^u_s)\|^2 ds + g(X^u_1)\right)
\label{eq:a_t}
\end{equation}

\subsubsection{Derivation of the Basic Adjoint Matching Objective}
Now we derive the basic adjoint matching objective
\begin{equation}
    \mathcal{L}_{\text{basic-AM}}(u^\theta; \bm{X}^{\bar{u}}) = \frac{1}{2} \int_0^1 \|u_t^\theta(X^{\bar{u}}_t) + a_t(\bm{X}^{\bar{u}},\bar{u})\|^2 dt, \quad\quad \bm{X}^{\bar{u}} \sim p^u, \,\, \bar{u}=\texttt{stopgrad}(u^\theta)
\label{eq:basic_AM_objective}
\end{equation}
The trajectories are generated by the neural network with fixed parameters ($\bar{u}$), detached from gradient computation during simulation. This avoids retaining the full simulation computation graph. It does not remove the costs of trajectory generation and network updates.

To derive this objective, we start from the optimal control objective
\begin{equation}
    \mathcal{L}(u) = \int_0^1 \frac{1}{2} \| u_t^\theta(X^u_t)\|^2 dt + g(X^u_1)
\end{equation}
We follow a similar procedure by adding the dynamics constraint with the adjoint state $a_t(\bm{X}^{\bar{u}},\bar{u})$ as its multiplier. Here, $a_t(\bm{X}^{\bar{u}},\bar{u})$ is independent of $\theta$.
\begin{align}
    \mathcal{L}(u) &= \int_0^1 \frac{1}{2} \| u_t^\theta(X^u_t)\|^2 dt + g(X^u_1) - \int_0^1 a_t^\top \left( dX^u_t - u_t^\theta(X_t^u)dt - \sigma_t dW_t \right) \\
    &= \int_0^1 \frac{1}{2} \| u_t^\theta(X^u_t)\|^2 dt + g(X^u_1) - a_1^\top X^u_1 + a_0^\top X^u_0 \notag \\
    & \quad + \int_0^1 (X_t^u)^\top \left( \frac{da_t}{dt} \right) dt + \int_0^1 a_t^\top u_t^\theta(X_t^u)dt + \int_0^1 \sigma_t a_t^\top dW_t
\end{align}
Differentiating with respect to the neural network parameter $\theta$ gives
\begin{align}
    \frac{d \mathcal{L}(u)}{d\theta} &= \frac{d}{d\theta}\Bigg( \int_0^1 \frac{1}{2} \| u_t^\theta(X^u_t)\|^2 dt + g(X^u_1) - a_1^\top X^u_1 + a_0^\top X^u_0 \notag \\
    & \quad + \int_0^1 (X_t^u)^\top \left( \frac{da_t}{dt} \right) dt + \int_0^1 a_t^\top u_t^\theta(X_t^u)dt + \int_0^1 \sigma_t a_t^\top dW_t \Bigg) \\
    &= \frac{d}{d\theta}\Bigg( \int_0^1 \frac{1}{2} \| u_t^\theta(X^u_t)\|^2 dt + g(X^u_1) - a_1^\top X^u_1 \notag \\
    & \quad + \int_0^1 (X_t^u)^\top \left( \frac{da_t}{dt} \right) dt + \int_0^1 a_t^\top u_t^\theta(X_t^u)dt \Bigg)
\end{align}

Here, $X_0$ is fixed and therefore independent of $\theta$.
Differentiating the first term gives
\begin{align}
    \frac{d}{d\theta}\left(\int_0^1\frac{1}{2}\|u_t^\theta(X_t^u)\|^2dt\right)
    &=\int_0^1\left(\frac{\partial u_t^\theta}{\partial\theta}(X_t^{\bar{u}})
    +\nabla_xu_t^\theta(X_t^{\bar{u}})\frac{\partial X_t^u}{\partial\theta}\right)^\top u_t^\theta(X_t^{\bar{u}})dt\\
    &=\int_0^1\left(\frac{\partial u_t^\theta}{\partial\theta}(X_t^{\bar{u}})\right)^\top u_t^\theta(X_t^{\bar{u}})
    +\left(\frac{\partial X_t^u}{\partial\theta}\right)^\top\nabla_x\left(\frac{1}{2}\|u_t^\theta(X_t^{\bar{u}})\|^2\right)dt.
\end{align}
The $X_1$ terms are
\begin{equation}
    \frac{d}{d\theta}\left(g(X_1^u)-a_1^\top X_1^u\right)
    =\left(\frac{\partial X_1^u}{\partial\theta}\right)^\top\left(\nabla_xg(X_1^{\bar{u}})-a_1\right)=0
    \qquad\left(\text{by \cref{eq:adjoint-terminal}}\right).
\end{equation}
The last two integrals are
\begin{align}
    &\frac{d}{d\theta}\left(\int_0^1(X_t^u)^\top\frac{da_t}{dt}dt+\int_0^1a_t^\top u_t^\theta(X_t^u)dt\right)\notag\\
    &=\int_0^1\left(\frac{\partial u_t^\theta}{\partial\theta}(X_t^{\bar{u}})\right)^\top a_tdt
    +\int_0^1\left(\frac{\partial X_t^u}{\partial\theta}\right)^\top
    \left(\frac{da_t}{dt}+\left(\nabla_xu_t^\theta(X_t^{\bar{u}})\right)^\top a_t\right)dt.
\end{align}
Gathering all the terms, we have
\begin{align}
    \frac{d\mathcal{L}(u)}{d\theta}
    &=\int_0^1\left(\frac{\partial u_t^\theta}{\partial\theta}(X_t^{\bar{u}})\right)^\top
    \left(u_t^\theta(X_t^{\bar{u}})+a_t\right)dt\notag\\
    &\quad+\int_0^1\left(\frac{\partial X_t^u}{\partial\theta}\right)^\top
    \left(\frac{da_t}{dt}+\left(\nabla_xu_t^\theta(X_t^{\bar{u}})\right)^\top a_t
    +\nabla_x\left(\frac{1}{2}\|u_t^\theta(X_t^{\bar{u}})\|^2\right)\right)dt\\
    &=\int_0^1\left(\frac{\partial u_t^\theta}{\partial\theta}(X_t^{\bar{u}})\right)^\top
    \left(u_t^\theta(X_t^{\bar{u}})+a_t\right)dt
    \qquad\left(\text{by \cref{eq:adjoint-ode}}\right).\label{eq:first_stopgrad}
\end{align}
Throughout this derivation, the $\frac{\partial X_t^u}{\partial \theta}$ terms account for the dependence of trajectories on $\theta$. These terms cancel by the adjoint equations. The remaining terms are evaluated using trajectories generated by $\bar{u} = \texttt{stopgrad}(u)$.
We can rewrite \cref{eq:first_stopgrad} as
\begin{align}
    \frac{d \mathcal{L}(u)}{d\theta} &= \int_0^1 \frac{\partial}{\partial \theta}\left( \frac{1}{2} \| u^\theta_t(X_t^{\bar{u}})\|^2 \right) + \frac{\partial}{\partial \theta} \left(a_t^\top u_t^\theta(X_t^{\bar{u}})\right) dt \\
    &= \frac{d}{d\theta} \left( \int_0^1  \frac{1}{2} \| u^\theta_t(X_t^{\bar{u}})\|^2 + a_t^\top u_t^\theta(X_t^{\bar{u}}) dt \right) \\
    &= \frac{d}{d\theta} \left( \int_0^1  \frac{1}{2} \| u^\theta_t(X_t^{\bar{u}})\|^2 + a_t^\top u_t^\theta(X_t^{\bar{u}}) + \frac{1}{2} \|a_t\|^2 dt \right) \quad (\because a_t \text{ is independent of } \theta)  \\
    &= \frac{d}{d\theta} \left( \int_0^1  \frac{1}{2} \| u^\theta_t(X_t^{\bar{u}}) + a_t(\bm{X}^{\bar{u}}, \bar{u}) \|^2 dt \right)
\end{align}
Thus, the basic AM objective in \cref{eq:basic_AM_objective} can be used to optimize $u^\theta_t$ using gradient descent.

\subsubsection{Connection to Optimal Control, the Lean Adjoint State, and the Adjoint Matching Objective}
The pathwise adjoint system applies to any sufficiently regular admissible control. We show the consistency of the functional stationarity condition with a verified SOC optimum. During the following variation, trajectories and adjoints are held fixed, and the perturbation is a state-feedback function $h_t=h_t(X_t)$.
The critical-point condition for adjoint matching follows from the first variation of
$\mathbb{E}_{p^u}[\mathcal{L}_{\text{basic-AM}}]$.
\begin{align}
    \frac{\delta \mathbb{E}_{p^u}[\mathcal{L}_{\text{basic-AM}}]}{\delta u} &= \left.\frac{d}{d\epsilon}\left( \mathbb{E}_{p^u}\left[ \frac{1}{2} \int_0^1 \|(u_t^\theta + \epsilon h_t)(X^{\bar{u}}_t) + a_t(\bm{X}^{\bar{u}})\|^2 dt \right] \right)\right|_{\epsilon=0} \\
    &= \mathbb{E}_{p^u}\left[ \int_0^1 h_t^\top(u_t^\theta(X^{\bar{u}}_t) + a_t(\bm{X}^{\bar{u}})) dt \right] \\
    &= \mathbb{E}_{p^u}\left[ \int_0^1 h_t^\top(u_t^\theta(X^{\bar{u}}_t) + \mathbb{E}_{\bm{X}\sim p^u}[a_t(\bm{X}^{\bar{u}})| X_t^{\bar{u}}] ) dt \right]
\end{align}
The last equality follows from the tower property, $\mathbb{E}_X[X] = \mathbb{E}_Y[\mathbb{E}_X[X|Y]]$. Thus, if $u^*$ satisfies the following condition, then $u^*$ is a critical point of the adjoint matching objective.
\begin{equation}
    u_t^*(x_t) = -\mathbb{E}_{\bm{X}^*\sim p^{u^*}}[a_t(\bm{X}^{u^*})| X_t^{u^*}=x_t]
\label{eq:AM-critical}
\end{equation}
For a verified SOC optimum $u^*$, the conditional expected cost-to-go equals $V_t$. Using \cref{eq:a_t} and interchanging differentiation and expectation gives
\begin{align}
    u_t^*(x_t)
    &=-\mathbb{E}_{p^{u^*}}\left[a_t(\bm{X}^{u^*})\mid X_t^{u^*}=x_t\right]\\
    &=-\mathbb{E}_{p^{u^*}}\left[\nabla_{x_t}\mathcal{J}(u^*,x_t)\mid X_t^{u^*}=x_t\right]\\
    &=-\nabla_{x_t}\mathbb{E}_{p^{u^*}}\left[\mathcal{J}(u^*,x_t)\mid X_t^{u^*}=x_t\right]\\
    &=-\nabla_{x_t}V_t(x_t).
\end{align}
Thus the SOC optimum satisfies the functional stationarity condition. This calculation does not establish global optimality of an arbitrary stationary point in neural-network parameters.

We can further simplify the adjoint system to the \textit{lean adjoint system}.
Multiplying \cref{eq:AM-critical} by $\left(\nabla_xu_t^*(x_t)\right)^\top$, we have
\begin{align}
    &\left(\nabla_xu_t^*(x_t)\right)^\top u_t^*(x_t)
    =-\mathbb{E}_{p^{u^*}}\left[\left(\nabla_xu_t^*(x_t)\right)^\top a_t(\bm{X}^{u^*})\mid X_t^{u^*}=x_t\right]\\
    \iff&\mathbb{E}_{p^{u^*}}\left[-\nabla_x\left(\frac{1}{2}\|u_t^*(x_t)\|^2\right)
    -\left(\nabla_xu_t^*(x_t)\right)^\top a_t(\bm{X}^{u^*})\mid X_t^{u^*}=x_t\right]=0.
\end{align}
The adjoint ODE has zero conditional mean drift at the SOC optimum. This motivates the lean adjoint below; it does not imply that each pathwise adjoint is constant. The functional critical-point result for lean Adjoint Matching is provided by \citet{domingo2024adjoint}, Proposition 7, under its stated assumptions.
\begin{subequations}
\begin{empheq}[left={\empheqlbrace}]{align}
\frac{d \tilde{a}_t}{dt} &= 0 \label{eq:lean-adjoint-ode} \\
\tilde{a}_1 &= \nabla_x g(X_1^u) \label{eq:lean-adjoint-terminal}
\end{empheq}
\label{eq:lean-adjoint-system}
\end{subequations}
where $\tilde{a}$ is called the \textit{lean adjoint state}. The lean adjoint state is no longer equivalent to the gradient of the cost-to-go function. Since the base SDE has no drift and the state cost is zero, the lean adjoint state is constant in time, with $\tilde{a}_t = \nabla_x g(X_1^{u})$.
The resulting adjoint matching objective is
\begin{align}
\begin{split}
    \mathcal{L}_{\text{AM}}(u^\theta; \bm{X}^{\bar{u}}) &= \frac{1}{2} \int_0^1 \|u_t^\theta(X^{\bar{u}}_t) + \tilde{a}_t(\bm{X}^{\bar{u}})\|^2 dt = \frac{1}{2} \int_0^1 \|u_t^\theta(X^{\bar{u}}_t) + \nabla_xg(X_1^{\bar{u}})\|^2 dt \\
    & \text{where} \quad \bm{X}^{\bar{u}} \sim p^u, \quad \bar{u}=\texttt{stopgrad}(u^\theta)
\end{split}
\label{eq:lean_AM_objective}
\end{align}
Replacing the terminal cost by $\rho g$ and taking expectation over trajectories gives \cref{eq:AM_objective}, up to the constant factor $1/2$.
The final adjoint matching algorithm is outlined in \cref{appendix:alg:AM}.

\begin{algorithm}[!htbp]
\caption{Adjoint Matching~\citep{domingo2024adjoint}}
\textbf{Require:} Parameterized drift $u^\theta$, differentiable terminal cost $g(x)$, nominal distribution $\hat{\mu}$, target distribution $\nu$, replay buffer $\mathcal{B}$, number of epochs $E$.

\textbf{Initialize:} $u^\theta$

\begin{algorithmic}[1]
\FOR{epoch in $1,2, \ldots, E$}
    \STATE Sample $\{\hat{X}^i_0\}_{i=1}^M \sim \hat{\mu}$
    \STATE Sample $\{t^i\}_{i=1}^M \sim \mathcal{U}[0,1]$.
    \STATE Simulate (\ref{eq:drsb_dyn}) to get $\{(\tilde{X}^i_0, t^i, X_t^i, X_1^i)\}_{i=1}^M$.
    \STATE Update replay buffer $\mathcal{B} \leftarrow \mathcal{B} \cup \{ (t^i , X^i_t,\, X_1^i)\}_{i=1}^M$
    \STATE Take gradient steps $\nabla_\theta \mathcal{L}_{\text{AM}}$ to minimize $\mathcal{L}_{\text{AM}}(\theta) = \mathbb{E}_{(t, X_t, X_1) \sim \mathcal{B}} \left[ \| u^\theta(t, X_t) + \nabla g(X_1) \|^2 \right]$.
\ENDFOR
\end{algorithmic}
\label{appendix:alg:AM}
\end{algorithm}

\clearpage
\section{Full DRSB Algorithm}
\label{appendix:full_algorithm}

The main DRSB training loop is given in \cref{alg:app_drsb}.
Each iteration consists of three steps: worst-case distribution update, terminal cost network update, and control network update.
DRSB-W and DRSB-S differ only in the worst-case subroutine.
\cref{alg:app_wasserstein_worst_case} details the Wasserstein worst-case optimization, which iteratively updates each sample via gradient ascent (\cref{eq:worst_case_gd}) and adjusts $\lambda$ via gradient descent (\cref{eq:lambda_update_rule}).
\cref{alg:app_sinkhorn_worst_case} details the Sinkhorn worst-case optimization, which computes $\lambda^*$ in closed form (\cref{eq:sdrsb_opt_lambda}) and samples from the Gaussian worst-case distribution (\cref{eq:sink_gaussian_worst}).
\begin{algorithm}[!htbp]
\small
\caption{Distributionally Robust Schr\"{o}dinger Bridge}
\textbf{Require:} Parameterized drift $u^\theta$, parameterized terminal cost $g^\phi$, nominal distribution $\hat{\mu}$, target distribution $\nu$, replay buffer $\mathcal{B}$, number of epochs $E$.

\textbf{Initialize:} $u^\theta$ with a pretrained SB model from $\hat{\mu}$ to $\nu$, $g^\phi$, $\lambda$

\begin{algorithmic}[1]
\FOR{epoch in $1,2, \ldots, E$}
    \STATE $\{\tilde{X}^i_0\}_{i=1}^M, \lambda \leftarrow$ \texttt{WassersteinWorstCase}($u^\theta, \lambda)$ (\cref{alg:app_wasserstein_worst_case})  \textit{or} \\
    $\{\tilde{X}^i_0\}_{i=1}^M \leftarrow$ \texttt{SinkhornWorstCase}($u^\theta)$ \\(\cref{alg:app_sinkhorn_worst_case})
    \STATE Sample $\{t^i\}_{i=1}^M \sim \mathcal{U}[0,1]$.
    \STATE Simulate (\ref{eq:drsb_dyn}) to get $\{(\tilde{X}^i_0, t^i, X_t^i, X_1^i)\}_{i=1}^M$.
    \STATE Update replay buffer $\mathcal{B} \leftarrow \mathcal{B} \cup \{ (t^i , X^i_t,\, X_1^i)\}_{i=1}^M$
    \STATE Take gradient steps $\nabla_\phi \mathcal{L}_{D}$ to minimize \\
    $\mathcal{L}_D(\phi)= \mathbb{E}_{X_1 \sim \mathcal{B}}\left[\log(1 + e^{-g^\phi(X_1)})\right] + \mathbb{E}_{X_1 \sim \nu}\left[\log (1 + e^{g^\phi(X_1)})\right]$
    \STATE Take gradient steps $\nabla_\theta \mathcal{L}_{\text{AM}}$ to minimize \\ $\mathcal{L}_{\text{AM}}(\theta) = \mathbb{E}_{(t, X_t, X_1) \sim \mathcal{B}} \left[ \| u^\theta(t, X_t) + \rho \nabla g^\phi(X_1) \|^2 \right]$.
\ENDFOR
\end{algorithmic}
\label{alg:app_drsb}
\end{algorithm}

\begin{algorithm}[!htbp]
\small
\caption{Wasserstein worst-case optimization}
\textbf{Require:} Minibatch size $\beta$, number of $\lambda$ updates $L$, worst-case update threshold $\tau$.

\textbf{Function }\texttt{WassersteinWorstCase}($u^\theta, \lambda$)

\begin{algorithmic}[1]
\STATE Prepare result set $\mathcal{R} = \{\}$
\FOR{$\lambda$ update iteration $1,2,\ldots, L$}
    \STATE Sample $\{\hat{X}^j_0\}_{j=1}^\beta \sim \hat{\mu}$
    \STATE $\{\tilde{X}^j_0\}_{j=1}^\beta \leftarrow \{\hat{X}^j_0\}_{j=1}^\beta$
    \REPEAT
        \STATE $\{X^j_0\}_{j=1}^\beta \leftarrow \{\tilde{X}^j_0\}_{j=1}^\beta$
        \STATE $\left\{ \tilde{X}^j_0 \leftarrow X^j_0 + \eta \left[ - u^\theta(0, X^j_0) - 2\lambda (X^j_0 - \hat{X}_0^j) \right] \right\}_{j=1}^{\beta}$
    \UNTIL{$\frac{1}{\beta}\sum_{j=1}^\beta \   \|\tilde{X}^j_0 - X^j_0\|^2 < \tau$.}
    \STATE Update $\lambda \leftarrow \lambda - \zeta \left[ \varepsilon - \frac{1}{\beta}\sum_{j=1}^\beta \   \|\tilde{X}^j_0 - \hat{X}^j_0\|^2 \right]$
    \STATE Save $\mathcal{R} \leftarrow \mathcal{R}\cup \{\tilde{X}^j_0\}_{j=1}^\beta$
\ENDFOR
\end{algorithmic}
\textbf{Return:} $\mathcal{R}=\{\tilde{X}^i_0\}_{i=1}^{L\beta}$, $\lambda$
\label{alg:app_wasserstein_worst_case}
\end{algorithm}

\begin{algorithm}[!htbp]
\small
\caption{Sinkhorn worst-case optimization}
\textbf{Require:} Batch size $B$, number of worst-case samples per clean sample $w$. \\
\textbf{Function} \texttt{SinkhornWorstCase}($u^\theta$)

\begin{algorithmic}[1]
\STATE Sample $\{\hat{X}^j_0\}_{j=1}^B \sim \hat{\mu}$
\STATE Compute $\lambda \leftarrow \sqrt{\left( \varepsilon + \frac{\kappa d}{2} \log \pi \kappa \right)^{-1} \frac{1}{B}\sum_{j=1}^B \frac{\| u^{\theta}(0, \hat{X}_0^j) \|^2}{4}}$
\STATE Sample $\left\{ \{\tilde{X}^l_0\}_{l=1}^w \sim \mathcal{N} \left( \hat{X}^j_0 - \frac{u^{\theta}_0(\hat{X}^j_0)}{2\lambda}, \frac{\kappa}{2} I \right) \right\}_{j=1}^B$
\end{algorithmic}
\textbf{Return:} $\{\tilde{X}^i_0\}_{i=1}^{wB}$
\label{alg:app_sinkhorn_worst_case}
\end{algorithm}

\FloatBarrier
\section{Implementation Details}
\label{appendix:imple_details}

All experiments are conducted on a single NVIDIA RTX 4090 GPU.
The drift network $u^\theta$ and the terminal cost network $g^\phi$ are both parameterized as multi-layer perceptrons.
For the 2D transport tasks (Gaussian and GMM), the baseline SB model is pretrained with the DSBM algorithm~\citep{shi2023diffusion}, and the DRSB training is initialized from this pretrained model following the successive approximation scheme in \cref{alg:drsb_overview}.
For the FFHQ image translation task, the model operates in the 512-dimensional latent space of the ALAE autoencoder~\citep{pidhorskyi2020adversarial}.

\subsection{Practical Techniques for Adversarial Training}

The minimax structure of DRSB can make training unstable, particularly in high-dimensional settings.
The following techniques stabilize the adversarial optimization.

\paragraph{Gradual adversarial training.}
The adversary in DRSB is the worst-case optimization, and its penalty is controlled by the dual variable $\lambda$.
Gradually decreasing $\lambda$ from a large initial value increases the adversarial shift slowly.
In DRSB-W, this occurs naturally through the $\lambda$ gradient descent steps (\cref{fig:lambda_opt}).
In DRSB-S, the closed-form $\lambda^*$ can produce extreme worst-case shifts from the start.
To mitigate this, we linearly increase the ambiguity set radius $\varepsilon$ over the first few iterations (\cref{fig:eps_schedule}).

\begin{figure}[htbp]
    \centering
    \begin{subfigure}[b]{0.4\linewidth}
        \centering
        \includegraphics[width=\linewidth]{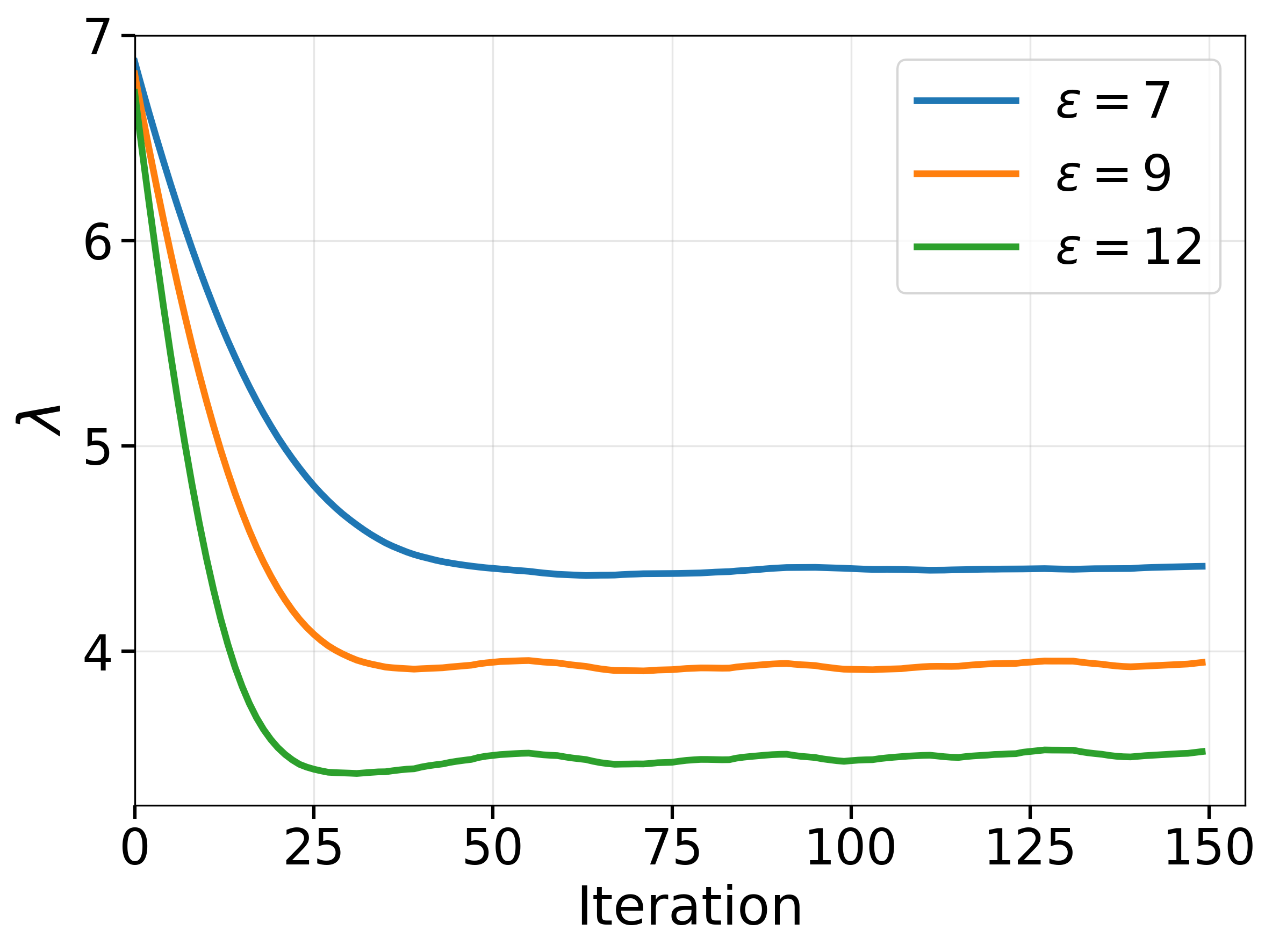}
        \caption{DRSB-W $\lambda$ evolution.}
        \label{fig:lambda_opt}
    \end{subfigure}
    \hfill
    \begin{subfigure}[b]{0.4\linewidth}
        \centering
        \includegraphics[width=\linewidth]{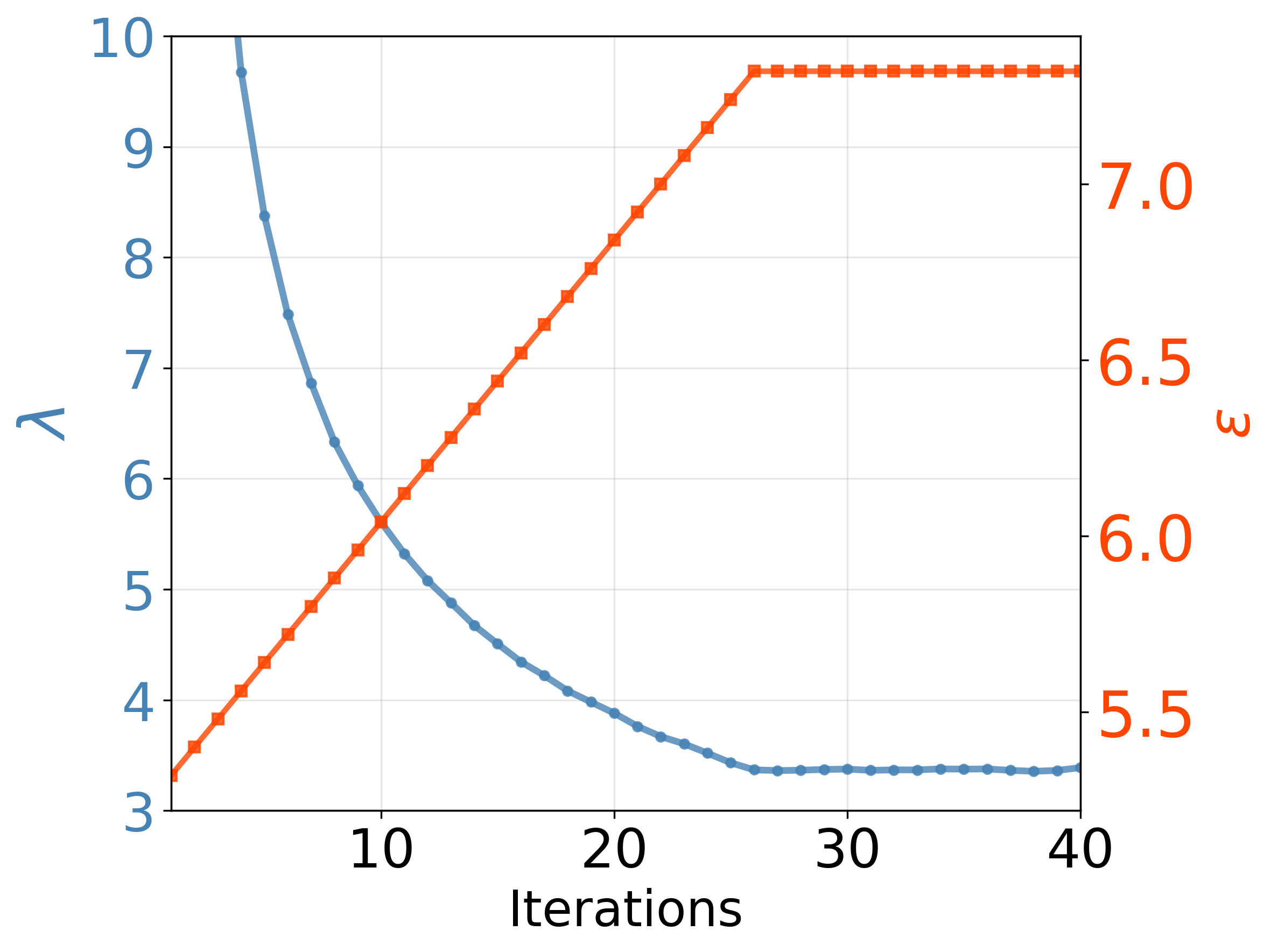}
        \caption{DRSB-S $\varepsilon$ schedule.}
        \label{fig:eps_schedule}
    \end{subfigure}
    \caption{Training dynamics. (a) Evolution of $\lambda$ in DRSB-W. (b) $\varepsilon$-schedule for DRSB-S and corresponding $\lambda$ evolution.}
    \label{fig:training_dynamics}
\end{figure}

\paragraph{Pretraining the terminal cost network $g^\phi$.}
Consistent with findings in robust control~\citep{bea1998successive}, proper initialization is critical for convergence in adversarial training.
We pretrain $g^\phi$ using the pretrained SB drift $\bar{u}$ by minimizing the joint objective $\mathcal{L}_D(\phi) + \xi \mathcal{L}_{\text{AM}}(\phi)$, where $\xi$ is a weight hyperparameter.
Generated samples for $\mathcal{L}_D(\phi)$ are obtained from $\bar{u}$, and $\mathcal{L}_{\text{AM}}(\phi)$ regresses $\rho \nabla g^\phi(X_1)$ to the negative drift $-\bar{u}(t, X_t)$.

\paragraph{Regularization with a pretrained SB model.}
Worst-case samples can push the controller away from the correct data manifold, leading to semantic degradation.
Since the pretrained SB control $\bar{u}$ captures the data manifold direction, we regularize $u^\theta$ to remain close to $\bar{u}$ via the loss $\mathcal{L}_{\text{AM}}(\theta) + \alpha \mathbb{E}[\|u^\theta - \bar{u}\|^2]$, where $\alpha$ is a weight hyperparameter.

\Needspace{26\baselineskip}
\subsection{Sensitivity to \texorpdfstring{$\kappa$}{kappa} in DRSB-S}

The entropy regularization parameter $\kappa$ in the Sinkhorn distance controls the support of the worst-case distribution.
As $\kappa \to 0$, the Sinkhorn distance approaches the Wasserstein distance, and the ambiguity set preserves geometric interpretability.
As $\kappa$ increases, higher-entropy couplings $\gamma$ are favored, and the worst-case distribution has wider support.

Training stability is sensitive to $\kappa$.
From \cref{eq:sink_gaussian_worst}, larger $\kappa$ produces a more diffuse worst-case distribution.
If $\kappa$ is too large, training becomes unstable due to extreme deviation from the nominal samples.

In high-dimensional settings, $\kappa$ interacts with the dimension $d$ through the optimal dual variable in \cref{eq:sdrsb_opt_lambda}.
The term $h(\kappa, \varepsilon) := \varepsilon + \frac{\kappa d}{2}\log(\pi\kappa)$ must satisfy $h(\kappa, \varepsilon) > 0$ for $\lambda^*$ to be well-defined.
If $h(\kappa, \varepsilon)$ is too large, $\lambda^*$ becomes small and the worst-case distribution deviates excessively from the nominal in \cref{eq:sink_gaussian_worst}.
Setting $\kappa$ near $0$ or $1/\pi$ eliminates the $\log(\pi\kappa)$ contribution and reduces the sensitivity to $d$.

\cref{tab:kappa_sensitivity} reports FID scores on the FFHQ Woman$\to$Man task for two values of $\kappa$ and several values of the ambiguity set radius $\varepsilon$.
Both $\kappa = 0.005$ and $\kappa = 0.32$ yield stable training and robustness to input noise.
In the FFHQ task with $d=512$, values of $\kappa$ far from $0$ or $1/\pi$ (e.g., $\kappa = 0.25$) or $\varepsilon > 30$ cause training to diverge.

\begin{table}[htbp]
\centering
\caption{FID ($\downarrow$) sensitivity to $\kappa$ and $\varepsilon$ on FFHQ Woman$\to$Man under noise $\delta$.}
\resizebox{0.8\linewidth}{!}{
\begin{tabular}{lcccc}
\toprule
\textbf{Model} & Org & $\delta=0.25$ & $0.50$ & $1.00$ \\
\midrule
DSBM & \textbf{25.55}$\pm$0.50 & 26.59$\pm$0.48 & 26.15$\pm$0.49 & 26.81$\pm$0.52 \\
\midrule
DRSB-S $\kappa$=0.005, $\varepsilon$=9 & 26.76$\pm$0.49 & 26.28$\pm$0.44 & 26.60$\pm$0.52 & 27.15$\pm$0.49 \\
DRSB-S $\kappa$=0.005, $\varepsilon$=12 & 26.36$\pm$0.48 & 26.48$\pm$0.48 & 26.10$\pm$0.50 & 26.66$\pm$0.52 \\
DRSB-S $\kappa$=0.005, $\varepsilon$=15 & 26.08$\pm$0.47 & 26.57$\pm$0.50 & 26.35$\pm$0.53 & 26.40$\pm$0.53 \\
\midrule
DRSB-S $\kappa$=0.32, $\varepsilon$=7 & 26.70$\pm$0.48 & 25.93$\pm$0.48 & 26.05$\pm$0.47 & 26.05$\pm$0.50 \\
DRSB-S $\kappa$=0.32, $\varepsilon$=12 & 26.52$\pm$0.50 & \textbf{25.82}$\pm$0.49 & \textbf{25.98}$\pm$0.49 & 25.94$\pm$0.47 \\
DRSB-S $\kappa$=0.32, $\varepsilon$=15 & 26.26$\pm$0.50 & 25.88$\pm$0.44 & 26.51$\pm$0.52 & \textbf{25.89}$\pm$0.49 \\
\bottomrule
\end{tabular}
}
\label{tab:kappa_sensitivity}
\end{table}

\clearpage
\section{Additional Experiments}
\label{appendix:additional_experiments}

The main text reports selected perturbation conditions in \cref{tab:2d_robustness} and \cref{tab:wdrsb_vs_sdrsb}.
This section provides the full results across all noise levels and shift directions.

\subsection{Noise Robustness}

\cref{tab:app_noise} reports SWD under $L_2$-norm noise at eight levels ($\delta = 0.5$ to $4.0$) for both tasks.
DSBM performs best at low noise in the Gaussian task, while DRSB-W and DRSB-S outperform DSBM at higher noise levels.
In the GMM task, DRSB-S outperforms both DSBM and DRSB-W at all noise levels except $\delta=0.5$.

\begin{table}[htbp]
\centering
\caption{SWD ($\downarrow$) on 2D transport tasks under varying $L_2$-norm noise levels $\delta$ (mean $\pm$ std over 10 seeds). \textbf{Bold} indicates best performance; \colorbox[gray]{0.9}{gray} indicates outperforming the baseline DSBM.}
\resizebox{\linewidth}{!}{
\begin{tabular}{lcccccccc}
\toprule
\textbf{Model} & $\delta$=0.5 & 1.0 & 1.5 & 2.0 & 2.5 & 3.0 & 3.5 & 4.0 \\
\midrule
\multicolumn{9}{c}{\textit{(a) Gaussian-to-Gaussian}} \\
\midrule
DSBM & \textbf{0.052}$\pm$0.006 & \textbf{0.072}$\pm$0.009 & \textbf{0.128}$\pm$0.012 & \textbf{0.214}$\pm$0.010 & \textbf{0.313}$\pm$0.015 & 0.431$\pm$0.012 & 0.555$\pm$0.013 & 0.684$\pm$0.013 \\
DRSB-W & 0.080$\pm$0.010 & 0.093$\pm$0.012 & 0.138$\pm$0.013 & 0.216$\pm$0.013 & 0.315$\pm$0.012 & \cellcolor[gray]{0.9}{0.425}$\pm$0.013 & \cellcolor[gray]{0.9}{0.548}$\pm$0.013 & \cellcolor[gray]{0.9}{0.678}$\pm$0.013 \\
DRSB-S & 0.131$\pm$0.011 & 0.134$\pm$0.014 & 0.165$\pm$0.015 & 0.227$\pm$0.016 & 0.314$\pm$0.015 & \textbf{0.414}$\pm$0.015 & \textbf{0.522}$\pm$0.015 & \textbf{0.638}$\pm$0.016 \\
\midrule
\multicolumn{9}{c}{\textit{(b) GMM-to-GMM}} \\
\midrule
DSBM & 1.164$\pm$0.137 & 1.287$\pm$0.112 & 1.597$\pm$0.077 & 1.842$\pm$0.086 & 2.051$\pm$0.116 & 2.175$\pm$0.118 & 2.254$\pm$0.108 & 2.407$\pm$0.243 \\
DRSB-W & \textbf{1.096}$\pm$0.216 & \cellcolor[gray]{0.9}{1.220}$\pm$0.172 & \cellcolor[gray]{0.9}{1.547}$\pm$0.114 & \cellcolor[gray]{0.9}{1.821}$\pm$0.123 & \cellcolor[gray]{0.9}{1.985}$\pm$0.111 & \cellcolor[gray]{0.9}{2.139}$\pm$0.128 & \cellcolor[gray]{0.9}{2.222}$\pm$0.166 & \cellcolor[gray]{0.9}{2.315}$\pm$0.236 \\
DRSB-S & 1.191$\pm$0.195 & \textbf{1.150}$\pm$0.206 & \textbf{1.293}$\pm$0.145 & \textbf{1.636}$\pm$0.157 & \textbf{1.836}$\pm$0.127 & \textbf{1.953}$\pm$0.095 & \textbf{2.013}$\pm$0.118 & \textbf{2.212}$\pm$0.413 \\
\bottomrule
\end{tabular}
}
\label{tab:app_noise}
\end{table}

\subsection{Shift Robustness}

\cref{tab:app_shift} reports SWD under directional distribution shifts and worst-case inputs.
DRSB-S achieves the lowest SWD under most shift directions in both tasks.
DRSB-W outperforms DSBM on the original GMM distribution, suggesting that the worst-case update can act as a regularizer on complex manifolds.

\begin{table}[htbp]
\centering
\caption{SWD ($\downarrow$) on 2D transport tasks under distribution shifts (mean $\pm$ std over 10 seeds). Org: original distribution; Up/Down/Left/Right: directional shift; W-wst: DRSB-W worst case; S-wst: DRSB-S worst case. \textbf{Bold} indicates best performance; \colorbox[gray]{0.9}{gray} indicates outperforming the baseline DSBM.}
\resizebox{\linewidth}{!}{%
\begin{tabular}{lccccccc}
\toprule
\textbf{Model} & Org & Up & Down & Left & Right & W-wst & S-wst \\
\midrule
\multicolumn{8}{c}{\textit{(a) Gaussian-to-Gaussian}} \\
\midrule
DSBM & \textbf{0.038}$\pm$0.007 & 0.665$\pm$0.012 & 0.673$\pm$0.020 & 0.666$\pm$0.011 & \textbf{0.651}$\pm$0.009 & 1.081$\pm$0.021 & 0.951$\pm$0.013 \\
DRSB-W & 0.144$\pm$0.007 & 0.691$\pm$0.020 & 0.698$\pm$0.020 & \cellcolor[gray]{0.9}{0.517}$\pm$0.014 & 0.792$\pm$0.014 & \cellcolor[gray]{0.9}{0.972}$\pm$0.007 & \cellcolor[gray]{0.9}{0.812}$\pm$0.012 \\
DRSB-S & 0.123$\pm$0.012 & \textbf{0.647}$\pm$0.018 & \textbf{0.640}$\pm$0.014 & \textbf{0.501}$\pm$0.014 & 0.769$\pm$0.016 & \textbf{0.881}$\pm$0.018 & \textbf{0.710}$\pm$0.013 \\
\midrule
\multicolumn{8}{c}{\textit{(b) GMM-to-GMM}} \\
\midrule
DSBM & 0.902$\pm$0.242 & 6.167$\pm$0.233 & 5.798$\pm$0.245 & 5.917$\pm$0.127 & 5.917$\pm$0.132 & 2.086$\pm$0.166 & 1.988$\pm$0.100 \\
DRSB-W & \textbf{0.885}$\pm$0.226 & \cellcolor[gray]{0.9}{5.949}$\pm$0.262 & 5.865$\pm$0.282 & 5.967$\pm$0.153 & \cellcolor[gray]{0.9}{5.785}$\pm$0.120 & \textbf{1.503}$\pm$0.052 & \cellcolor[gray]{0.9}{1.946}$\pm$0.100 \\
DRSB-S & 1.165$\pm$0.138 & \textbf{5.887}$\pm$0.247 & \textbf{5.698}$\pm$0.203 & \textbf{5.840}$\pm$0.095 & \textbf{5.675}$\pm$0.185 & \cellcolor[gray]{0.9}{1.791}$\pm$0.102 & \textbf{1.478}$\pm$0.100 \\
\bottomrule
\end{tabular}
}
\label{tab:app_shift}
\end{table}

\clearpage
\section{Worst-Case Cost Analysis}
\label{sec:worst_case_analysis}

\subsection{Worst-Case Cost Interpretation}
\label{sec:worst_case_interpretation}

\begin{remark}[Fixed-terminal-cost bound]
Fix a controller $u$ and terminal cost $g$. If $\tilde{\mu}$ globally maximizes $\mathbb{E}_{X_0\sim\mu}[V(u,X_0;g)]$ over $\mathfrak{M}_\epsilon(\hat{\mu})$, then
\begin{equation}
\forall\mu\in\mathfrak{M}_\epsilon(\hat{\mu}),\qquad
\mathbb{E}_{X_0\sim\mu}[V(u,X_0;g)]
\leq\mathbb{E}_{X_0\sim\tilde{\mu}}[V(u,X_0;g)].
\label{eq:cost_bound}
\end{equation}
This is a direct consequence of global inner optimality. It does not require an optimal controller or imply that the ambiguity constraint is active. Approximate adversarial updates need not attain this bound. Since $g$ is fixed, the inequality is not a certificate for the full objective with terminal KL divergence.
\end{remark}

Applying this bound to a test distribution requires verifying its membership in the ambiguity set. For empirical training and test distributions, the Wasserstein distance is well defined and can be computed from their samples. With the squared Euclidean transport cost, the Wasserstein ambiguity set constrains the average squared displacement from the nominal distribution.

\subsection{Empirical Cost Comparison}
\label{sec:empirical_cost_comparison}

We compare the measured costs of the computed adversarial distribution, four shifted distributions (Up, Down, Left, Right), and the nominal distribution within the Wasserstein ambiguity set. For fixed $u^\theta$ and $g^\phi$, we estimate $\mathbb{E}[\sum_t \tfrac12\|u_t\|^2\Delta t+\rho g^\phi(X_1)]$ over 1000 uniform time steps. In \cref{fig:cost_bound}, the computed adversarial distribution has the largest measured cost among these distributions. This finite comparison does not verify global inner optimality or certify the full objective with terminal KL divergence.

The control cost in \cref{fig:cost_bound} saturates as $\rho$ increases, while the terminal-cost term receives greater weight.

\begin{figure}[htbp]
\centering
\includegraphics[width=0.54\linewidth]{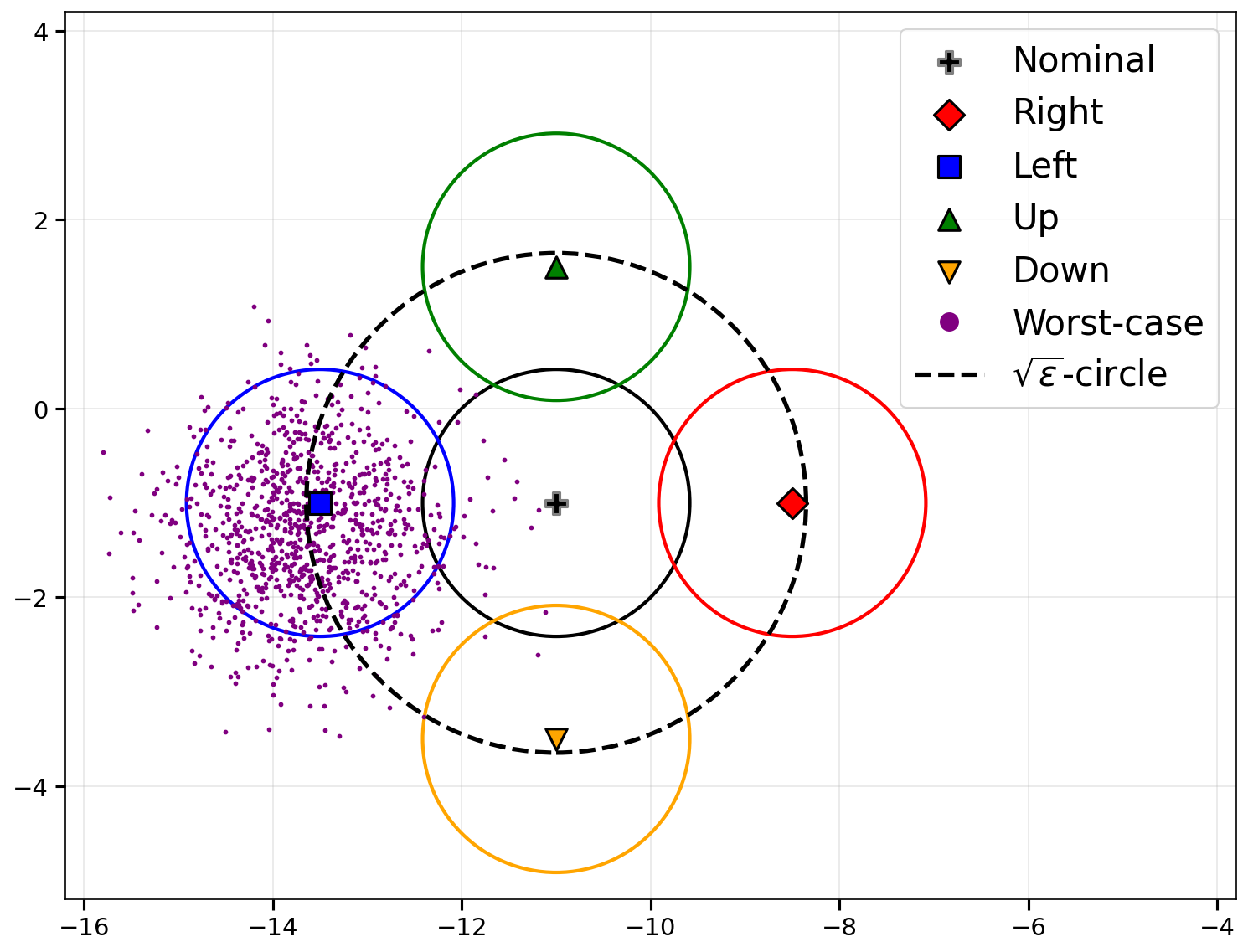}
\includegraphics[width=0.34\linewidth]{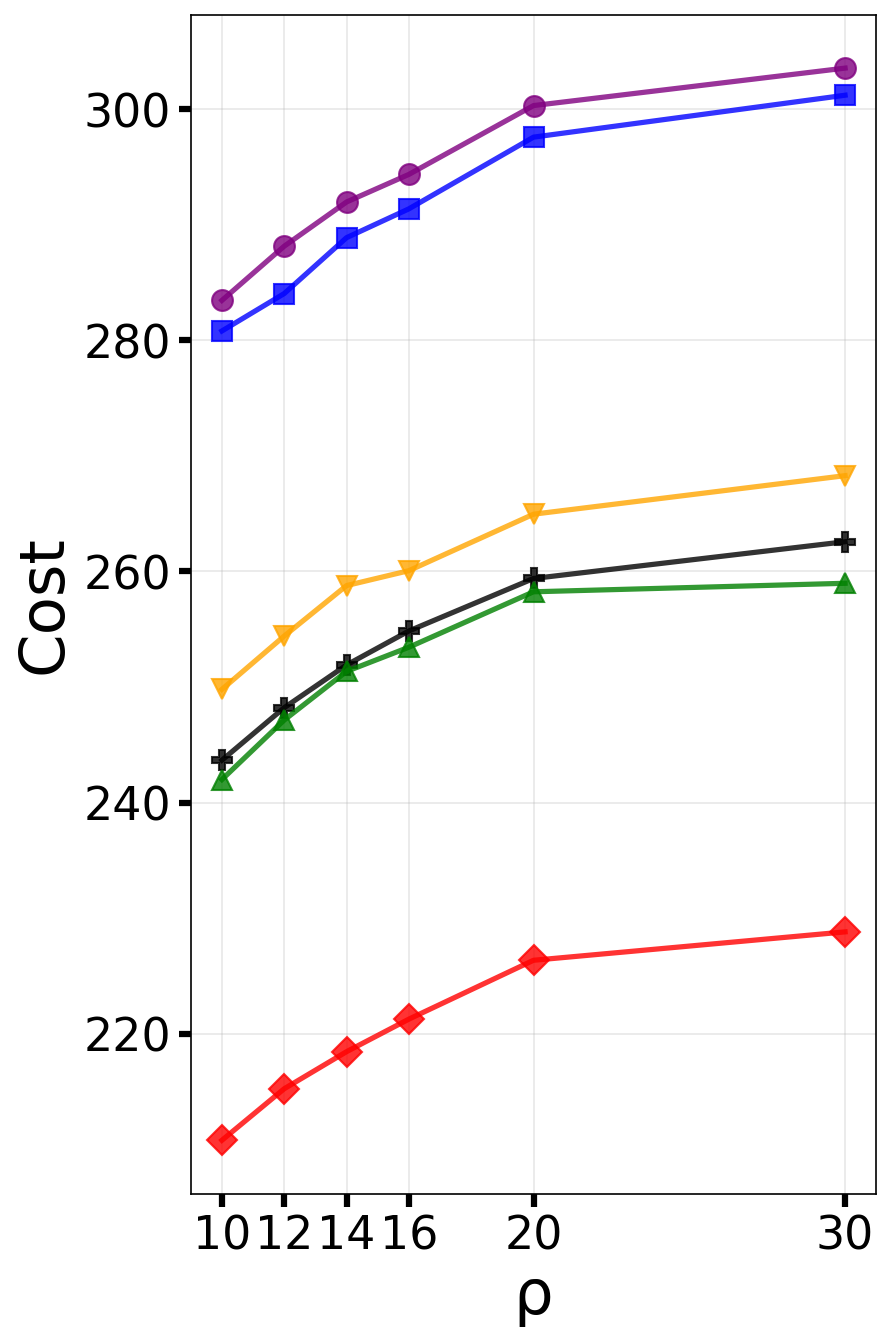}
\caption{Cost comparison within the Wasserstein ambiguity set. \textbf{Left}: Evaluated distributions with the ambiguity-set boundary (dashed). \textbf{Right}: Measured cost versus $\rho$.}
\label{fig:cost_bound}
\label{fig:distribution_visualization}
\label{fig:total_cost_vs_rho}
\end{figure}

\clearpage
\section{Additional Experimental Analysis}
\label{sec:rebuttal_analysis}

\subsection{Measured Computation Cost}
\label{sec:rebuttal_compute}

DRSB starts from a pretrained DSBM drift. \Cref{tab:rebuttal_compute} separates the two stages. The DRSB stage costs approximately $0.54\times$ the Gaussian DSBM time and $0.16\times$ the FFHQ DSBM time. Including pretraining gives approximately $1.54\times$ and $1.16\times$, respectively. Gradient updates account for about 80\% of Gaussian training time and over 90\% on FFHQ, while worst-case optimization accounts for less than 1.5\%. DRSB-W uses up to 90 worst-case gradient steps in these runs; DRSB-S uses Gaussian sampling with a closed-form $\lambda$ update under the linear approximation. FFHQ DRSB draws four Monte Carlo samples per clean sample. These measurements describe the tested workloads and do not establish how runtime scales with dimension.

\begin{table}[htbp]
\centering
\caption{Measured training costs. DRSB columns report the additional stage after DSBM pretraining. Trajectory counts follow the reported training configurations.}
\label{tab:rebuttal_compute}
\small
\begin{tabular}{lrrrr}
\toprule
& Gaussian DSBM & Gaussian DRSB & FFHQ DSBM & FFHQ DRSB \\
\midrule
Iterations & 140 & 50 & 70 & 50 \\
Gradient steps (total) & 224k & 101k & 3.36M & 320k \\
Trajectories & 4,096 & 1,024 & 20,480 & 40,960 \\
Wall-clock time & 754\,s & 404\,s & 8.5\,h & 1.4\,h \\
Stage time / DSBM time & $1\times$ & $0.54\times$ & $1\times$ & $0.16\times$ \\
\bottomrule
\end{tabular}
\end{table}

\subsection{Empirical Convergence}
\label{sec:rebuttal_convergence}

We examine the 2D Gaussian task with $\rho=10$, $\epsilon=12$, and 50 epochs, using a cosine learning-rate schedule from $3\times10^{-7}$ to $10^{-7}$. In \cref{tab:rebuttal_diagnostics}, $R_k=\left|1-\frac{1}{N\epsilon}\sum_i\|X_0^{i*}-\hat X_0^i\|^2\right|$ is the constraint residual, $G_k=\mathbb{E}_{p_1^{u^{(k)},\mu^{(k)}}}[g^\phi(X_1)]$ is the mean terminal critic output, and $U_k$ is the mean squared control magnitude on a fixed clean evaluation set. At the reported epochs $k\geq25$, $R_k<0.005$. The final dual variable is $\lambda_{50}=3.483$, close to the value 3.496 predicted by the active-constraint SOC relation. The terminal critic output and control magnitude change less at later epochs. These summaries measure stabilization, not gradients or stationarity residuals for all three blocks.

\begin{table}[htbp]
\centering
\caption{Reported update diagnostics on the 2D Gaussian task.}
\label{tab:rebuttal_diagnostics}
\small
\setlength{\tabcolsep}{4pt}
\begin{tabular}{lrrrrrrrrrr}
\toprule
$k$ & 5 & 10 & 15 & 20 & 25 & 30 & 35 & 40 & 45 & 50 \\
\midrule
$R_k$ & 0.260 & 0.319 & 0.240 & 0.111 & 0.002 & 0.004 & 0.002 & 0.000 & 0.002 & 0.001 \\
$\lambda_k$ & 6.601 & 5.696 & 4.535 & 3.627 & 3.427 & 3.457 & 3.477 & 3.488 & 3.494 & 3.483 \\
$G_k$ & 0.752 & 1.232 & 1.735 & 2.142 & 2.068 & 1.757 & 1.672 & 1.631 & 1.524 & 1.550 \\
$U_k$ & 22.94 & 22.92 & 23.19 & 23.74 & 24.10 & 24.22 & 24.29 & 24.36 & 24.32 & 24.30 \\
\bottomrule
\end{tabular}
\end{table}

We also evaluate the variational objective
\begin{equation}
\mathcal{L}(u,\mu,g)=\mathbb{E}_{X_0\sim\mu}[V(u,X_0;g)]-\rho\log\mathbb{E}_{X_1\sim\nu}[e^{g(X_1)}],
\end{equation}
where $V(u,x;g)$ includes the control effort and terminal cost $\rho g$. The control effort is estimated with a Riemann sum. The sampled values in \cref{tab:rebuttal_objective} stabilize near 312 after an observed peak of 313.44 at $k=32$. A separate run with constant learning rate $10^{-7}$ for 120 iterations also settles near 312. This provides an additional check beyond the decreasing learning-rate schedule, without proving convergence of the block algorithm.

\begin{table}[htbp]
\centering
\caption{Variational objective at the reported epochs.}
\label{tab:rebuttal_objective}
\small
\begin{tabular}{lrrrrrr}
\toprule
$k$ & 5 & 10 & 20 & 30 & 40 & 50 \\
\midrule
$\mathcal{L}$ & 270.99 & 276.16 & 304.99 & 311.99 & 311.44 & 311.37 \\
\bottomrule
\end{tabular}
\end{table}

\subsection{Empirical Check of the SOC Optimality Condition}
\label{sec:rebuttal_soc}

The simplified adversarial update replaces $\nabla_{X_0}V(u,X_0;g)$ by $-u^\theta(0,X_0)$ using \cref{eq:optimal_control}. The identity is exact for an optimal controller of the fixed-terminal-cost SOC problem, but finite training need not satisfy it. We assess this approximation on the same 2D Gaussian experiment by computing $\nabla_{X_0}V$ with a discrete adjoint recursion requiring one vector-Jacobian product per time step. The controller and terminal cost are held fixed for this gradient evaluation.

\Cref{tab:rebuttal_soc} reports both direction cosines and gradient-to-control magnitude ratios. The cosine tests whether the surrogate points in the cost-gradient direction, while the ratio tests its scale. Across the evaluated epochs, cosines are 0.9996--0.9998 and ratios are 0.9310--1.0761. At epoch 50, the cosine is 0.9998 and the ratio is 1.0089, corresponding to a gradient magnitude approximately 0.9\% larger than the control magnitude. These measurements support approximate satisfaction of the SOC relation on the evaluated samples. They do not establish exact SOC optimality or its validity across other tasks and input distributions.

\begin{table}[htbp]
\centering
\caption{SOC consistency diagnostics on the 2D Gaussian task.}
\label{tab:rebuttal_soc}
\small
\begin{tabular}{lrrrrr}
\toprule
$k$ & 10 & 20 & 30 & 40 & 50 \\
\midrule
$\cos\left(\nabla_{X_0}V,-u^\theta(0,X_0)\right)$ & 0.9997 & 0.9996 & 0.9996 & 0.9997 & 0.9998 \\
$\|\nabla_{X_0}V\|/\|u^\theta(0,X_0)\|$ & 0.9310 & 1.0761 & 1.0249 & 1.0278 & 1.0089 \\
\bottomrule
\end{tabular}
\end{table}

\end{document}